\documentclass{article}

\PassOptionsToPackage{numbers, compress}{natbib}

\usepackage[final]{neurips_2025}

\usepackage[utf8]{inputenc} 
\usepackage[T1]{fontenc}    
\usepackage[colorlinks,
            linkcolor=red,
            anchorcolor=green,
            citecolor=blue
            ]{hyperref} 
\usepackage{url}            
\usepackage{booktabs}       
\usepackage{amsfonts}       
\usepackage{nicefrac}       
\usepackage{microtype}      
\usepackage{xcolor}         

\usepackage{amsmath}
\usepackage{amssymb}
\usepackage{mathtools}
\usepackage{amsthm}
\usepackage{xspace}
\usepackage{enumitem}
\usepackage{xcolor}
\usepackage{thm-restate}
\usepackage{wrapfig}

\usepackage{pifont}
\usepackage[capitalize,noabbrev]{cleveref}

\theoremstyle{plain}
\newtheorem{theorem}{Theorem}[section]

\theoremstyle{definition}
\newtheorem{definition}[theorem]{Definition}

\theoremstyle{remark}
\newtheorem{remark}[theorem]{Remark}

\newcommand{\method}{RePO\xspace}
\newcommand{\newmethod}{RePO++\xspace}
\newcommand{\ie}{\emph{i.e., }}
\newcommand{\eg}{\emph{e.g., }}

\newcommand{\cf}{\emph{cf. }}

\newcommand{\rebuttal}[1]{{{#1}}}
\usepackage{multirow}
\title{RePO: Understanding Preference Learning Through ReLU-Based Optimization}

\author{
    Junkang Wu$^{1}$\thanks{Work done at Alibaba Group.}~~Kexin Huang$^{1}$~~Xue Wang$^{2}$~~Jinyang Gao$^{2}$~~Bolin Ding$^{2}$\\
    \textbf{Jiancan Wu$^{1,3}$~~Xiangnan He$^{4}$\thanks{Xiangnan He and Xiang Wang are the corresponding authors.}~~Xiang Wang$^{1}$\footnotemark[2]}\\
    $^{1}$University of Science and Technology of China, $^{2}$Alibaba Group \\
    $^{3}$Institute of Dataspace, Hefei Comprehensive National Science Center \\
    $^{4}$MoE Key Lab of BIPC, University of Science and Technology of China \\
    \texttt{\{jkwu0909, xiangwang1223, xiangnanhe\}@gmail.com}
}

\begin{document}

\maketitle

\begin{abstract}
    Preference learning has become a common approach in various recent methods for aligning large language models with human values. These methods optimize the preference margin between chosen and rejected responses, subject to certain constraints for avoiding over-optimization. In this paper, we report surprising empirical findings that simple ReLU activation can learn meaningful alignments even using \emph{none} of the following: (i) sigmoid-based gradient constraints, (ii) explicit regularization terms. Our experiments show that over-optimization does exist, but a threshold parameter $\gamma$ plays an essential role in preventing it by dynamically filtering training examples. 
    We further provide theoretical analysis demonstrating that ReLU-based Preference Optimization (RePO) corresponds to the convex envelope of the 0-1 loss, establishing its fundamental soundness.
    Our ``RePO'' method achieves competitive or superior results compared to established preference optimization approaches. We hope this simple baseline will motivate researchers to rethink the fundamental mechanisms behind preference optimization for language model alignment.
\end{abstract}

\section{Introduction}
\label{sec:intro}

Recent years have witnessed significant advances in aligning large language models (LLMs) with human preferences \cite{Gemmaini23,Llama2,OpenAI23,Bubeck23}. A primary approach, Reinforcement Learning from Human Feedback (RLHF) \cite{PPO}, first trains a reward model on preference data and then optimizes the LLM via reinforcement learning. While effective, RLHF's computational costs and training instability \cite{DPO,SlicHF} have motivated simpler offline alternatives like Direct Preference Optimization (DPO) \cite{DPO}, which bypasses explicit reward modeling.
Take DPO as a representative example: it optimizes the alignment margin between a preferred and a less-preferred response to the same prompt, as Figure \ref{fig:teaser} shows. The alignment of each response is quantified via an implicit reward, defined as the log-ratio of the predicted likelihoods under the policy model (\ie the LLM being optimized) and a reference model (\eg a fixed supervised fine-tuned (SFT) model).

A fundamental challenge in preference learning is \emph{over-optimization} --- where models excessively amplify reward margins between preferred and non-preferred responses, potentially degrading generation quality \cite{AmodeiOSCSM16,GaoSH23,DuboisLTZGBGLH23}. Several approaches have been developed to mitigate this issue. 
DPO \cite{DPO} and SimPO \cite{SimPO} employ sigmoid weighting through log-sigmoid activation that diminishes gradients as reward margins increase, naturally preventing over-optimization. The $\beta$ parameter controls gradient distribution sharpness --- larger values produce more binary-like gradients, as illustrated in Figure \ref{fig:intro_gradient}.
SLiC-HF \cite{SlicHF} addresses over-optimization differently by incorporating an SFT regularization term that anchors the model to its initial policy \cite{reward_hacking}, preventing excessive drift toward maximizing preference signals. 
These mechanisms effectively balance preference optimization with generation quality preservation, forming the foundation of current preference learning approaches.

Here, we present a surprising empirical finding: a simple ReLU activation can work well with \emph{none} of the above strategies for mitigating over-optimization. Our analysis reveals that as parameter $\beta$ in SimPO approaches infinity, its sigmoid weighting naturally converges to a binary thresholding mechanism --- motivating our exploration of \textbf{Re}LU-based \textbf{P}reference \textbf{O}ptimization (\textbf{RePO}). 
This mechanism uses a single ReLU function with only one hyperparameter $\gamma$, creating a clear decision boundary that selectively updates sample pairs with insufficient reward margins ($M_\theta < \gamma$) while filtering out well-separated pairs ($M_\theta \geq \gamma$). 
 We illustrate this ``RePO'' method in Figure~\ref{fig:teaser}.

\begin{figure*}[t]
    \begin{center}
        \includegraphics[width=1.0\columnwidth]{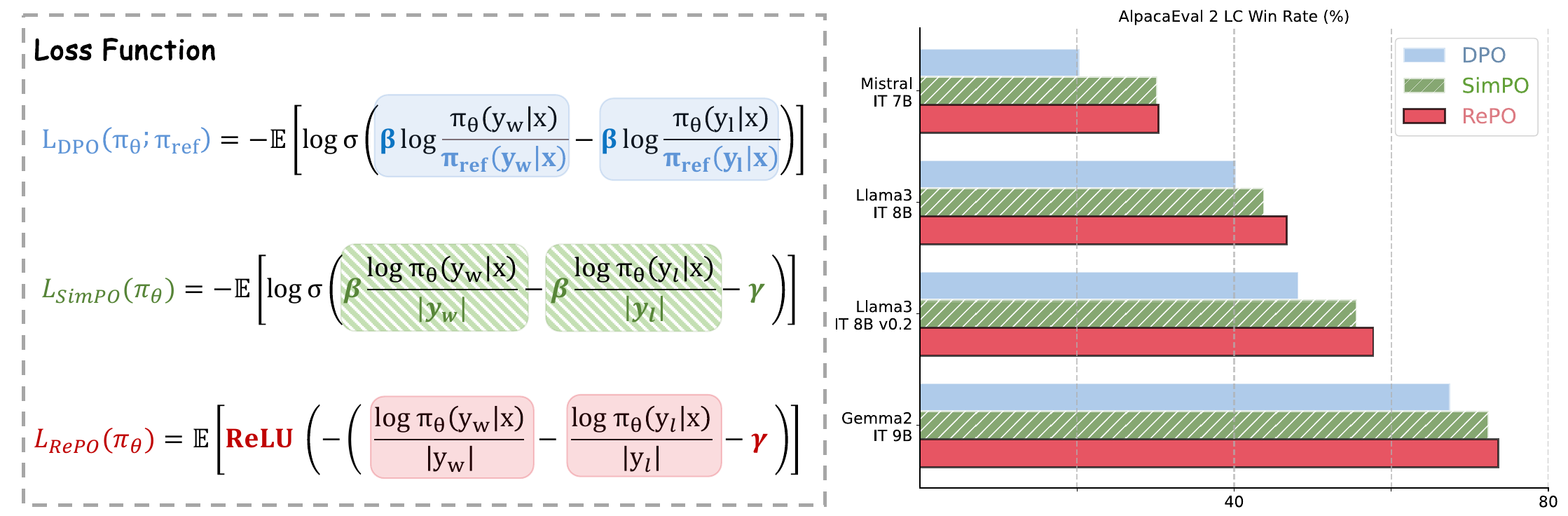}
        \caption{
            \textbf{Comparing preference learning mechanisms.} RePO employs a simpler binary thresholding mechanism than SimPO and DPO, as highlighted in the shaded box. Despite its simplicity, this mechanism achieves competitive results by naturally preventing over-optimization.
        }
        \label{fig:teaser}
    \end{center}
    \vspace{-20pt}
\end{figure*}


\definecolor{linecolor}{RGB}{220,122,130}
\begin{wrapfigure}{r}{0.55\textwidth}
    \centering
    \!\!\!\!\!\!\!\! \includegraphics[width=0.55\textwidth]{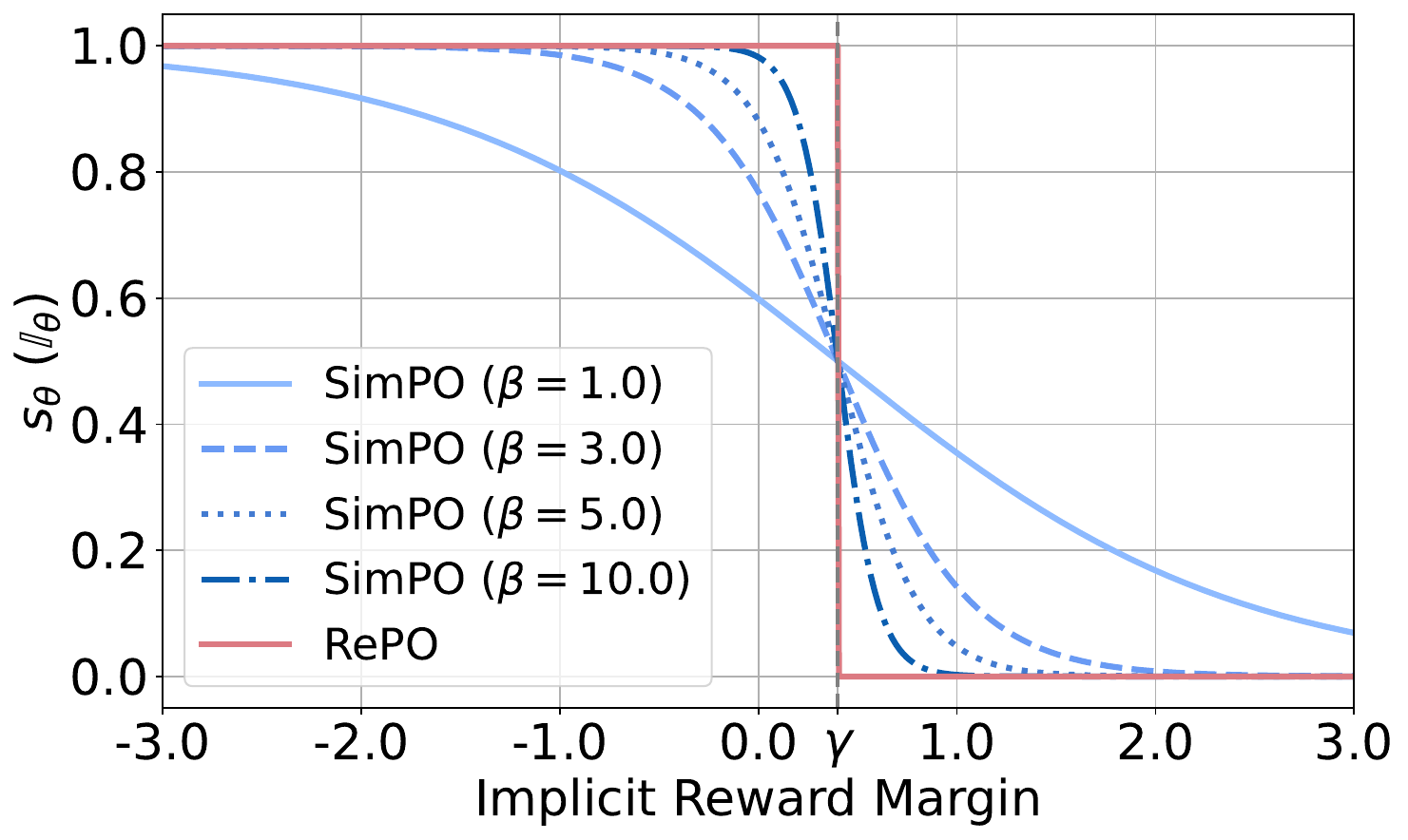}
    \caption{
    Gradient weighting functions of SimPO ($s_\theta$) and \method{} ($\mathbb{I}(M_\theta < \gamma)$). 
    As $\beta \to \infty$, $s_\theta$ converges to the binary indicator (\textcolor{linecolor}{red line}), establishing \method{} as the limit case of SimPO.
    }
    \label{fig:intro_gradient}
\end{wrapfigure}
 Thanks to its conceptual simplicity, RePO can serve as a hub that relates several existing methods. In essence, our method can be viewed as ``SimPO \emph{without} log-sigmoid'' or ``SLiC-HF \emph{without} SFT regularization term''. Interestingly, RePO is related to each method by removing one of its core components. Even so, RePO effectively prevents over-optimization while performing competitively or better (\cf Figure \ref{fig:teaser}).

We empirically show that over-optimization \emph{do} exist, but ReLU activation with threshold $\gamma$ is critical to prevent such solutions. 
This implies that in over-optimization regimes, selecting \emph{which} examples to learn from is more critical than determining \emph{how much} to learn from each.
The $\gamma$ threshold induces an emergent data filtering behavior, focusing dynamically on challenging samples relative to the model's current capability.
Our theoretical analysis reveals that RePO's ReLU loss corresponds precisely to the \textit{convex envelope} of the 0-1 loss 
(Theorem~\ref{thm:convex_env}), explaining why such a simple mechanism is so effective. 

Our simple baseline suggests that the ReLU activation with a proper threshold $\gamma$ can be an essential reason for the common success of related methods. We believe this work's significance lies in revealing how preference learning principles may be simpler than previously thought. By questioning conventional wisdom about necessary components, we hope to motivate researchers to reconsider the fundamental mechanisms behind preference optimization.

\section{Preliminaries}
\label{sec:preliminaries}

\textbf{Directed Preference Optimization (DPO).}
DPO~\citep{DPO} stands out as a leading method for offline preference optimization by eliminating the need for an explicit reward model. Instead, it reformulates the reward \(r(x,y)\) as a closed-form expression based on policy ratios:
\begin{equation}
    r(x,y) = \beta \log \frac{\pi_\theta(y \mid x)}{\pi_{\text{ref}}(y \mid x)} + \beta \log Z(x),
\end{equation}
where \(Z(x)\) is a partition function that does not depend on \(y\). This leads to the DPO loss for a given triplet \((x, y_w, y_l)\) as:
\begin{equation}
    \begin{aligned}
        \mathcal{L}_{\text{DPO}}(\pi_\theta; \pi_{\text{ref}}) =  - \mathbb{E}_{(x, y_w, y_l) \in \mathcal{D}} \left[ \log \sigma \left( \beta \left( \log \frac{\pi_\theta(y_w \mid x)}{\pi_\theta(y_l \mid x)} - \log \frac{\pi_{\text{ref}}(y_w \mid x)}{\pi_{\text{ref}}(y_l \mid x)} \right) \right) \right],
    \end{aligned}
\end{equation}
where \(\sigma(\cdot)\) denotes the sigmoid function. This loss encourages the policy \(\pi_\theta\) to prefer \(y_w\) over \(y_l\) in alignment with the reference policy.
\rebuttal{

\textbf{Sequence Likelihood Calibration (SLiC-HF).}
SLiC-HF~\citep{SlicHF} advances preference optimization with two key innovations: (1) it employs a sequence-level calibration loss that contrasts the log-probability difference between preferred and dispreferred responses using a margin $\gamma$, and (2) it integrates a regularization term to prevent divergence from the SFT policy, avoiding the need for an explicit KL penalty.
The SLiC-HF loss function is defined as:
\begin{equation}
    \begin{aligned}
        \mathcal{L}_{\text{SLiC-HF}}(\pi_\theta) = &\mathbb{E}_{(x, y_w, y_l) \in \mathcal{D}} \Big[ \text{ReLU} \Big( - \Big(  \log {\pi_\theta(y_w \mid x)} - \log \pi_\theta (y_l \mid x ) - \gamma \Big) \Big) - \lambda \log {\pi_\theta(y_w \mid x)} \Big].
    \end{aligned}
\end{equation}
}

\textbf{Simple Preference Optimization (SimPO).}
SimPO~\citep{SimPO} advances preference optimization with two key innovations: (1) it normalizes the reward by the length of the response, calculating the average log-probability per token for a response under the policy \(\pi_\theta\), and (2) it incorporates a target reward margin \(\gamma\) to ensure that the reward difference between the preferred and less preferred responses exceeds this margin. The SimPO loss function is defined as:
\begin{equation}
    \begin{aligned}
        \mathcal{L}_{\text{SimPO}}(\pi_\theta) = - \mathbb{E}_{(x, y_w, y_l) \in \mathcal{D}}\left[ \log \sigma  \left( \beta \left( \frac{\log \pi_\theta(y_w \mid x)}{|y_w|} - \frac{\log \pi_\theta(y_l \mid x)}{|y_l|}  - \gamma \right) \right) \right],
    \end{aligned}
\end{equation}
where \(|y|\) denotes the number of tokens in response \(y\), ensuring length-aware scaling of rewards, and \(\gamma\) is the predefined margin that enforces a minimum difference in rewards between \(y_w\) and \(y_l\). To align with subsequent discussions, we modify the original SimPO formulation by setting $\gamma$ to $\gamma / \beta$.

\definecolor{BrickRed}{rgb}{0.8, 0.25, 0.33}
\definecolor{ForestGreen}{rgb}{0.13, 0.55, 0.13}

\section{Exploring Simple ReLU Activation in Preference Learning}
\label{sec:analysis}
In this section, we explore what makes a simple ReLU activation function effective for preference learning. We first examine the surprising relationship between ReLU activation and sigmoid weighting through empirical experiments. Then, we investigate the key properties that emerge from this simple mechanism, specifically through the lens of gradient behavior, data filtering patterns, and over-optimization control.
\subsection{Examining ReLU-based Preference Optimization}
\label{sec:method}

\textbf{Simplification exploration.}
Our exploration began by questioning whether log-sigmoid activation or SFT regularization are truly necessary for mitigating over-optimization. We simplified the SimPO loss function through two key modifications: (i) removing the hyperparameter $\beta$, and (ii) replacing the log-sigmoid function with a ReLU activation.

We adopt the length-normalized \textit{implicit reward margin} $M_\theta$ (as introduced in SimPO \citep{SimPO}):
\begin{equation}
M_\theta = \frac{\log \pi_\theta(y_w \mid x)}{|y_w|} - \frac{\log \pi_\theta(y_l \mid x)}{|y_l|},
\end{equation}
which quantifies the policy's preference between responses.
Using $M_\theta$, we examine a loss function with the following form:
\begin{equation}
    \mathcal{L}_{\text{\method}}(\pi_\theta) = \mathbb{E}_{(x, y_w, y_l) \in \mathcal{D}} \left[ \mathrm{ReLU} \left( - (M_\theta - \gamma) \right) \right],
    \label{eq:repo}
\end{equation}
where $\gamma \in [0,1]$ is the sole hyperparameter representing the \textit{target reward margin}. 

\textbf{Gradient behavior investigation.}
We examine the gradient dynamics of \method{} and SimPO to reveal how our simplified approach addresses over-optimization:
\begin{equation}
    \nabla_\theta \mathcal{L}_{\text{SimPO}}(\pi_\theta) = -\beta \mathbb{E}_{\mathcal{D}} \left[ s_\theta \cdot \left( \nabla_{\theta, y_w} - \nabla_{\theta, y_l} \right) \right],\label{eq:simpo_grad}
\end{equation}
\begin{equation}
    \nabla_\theta \mathcal{L}_{\text{\method}}(\pi_\theta) = -\mathbb{E}_{\mathcal{D}} \left[ \mathbb{I}{(M_\theta < \gamma)} \cdot \left( \nabla_{\theta, y_w} - \nabla_{\theta, y_l} \right) \right],\label{eq:repo_grad}
\end{equation}

where $s_\theta = \sigma(\beta(-M_\theta + \gamma))$ is SimPO's sigmoid weighting function. 
The terms $\nabla_{\theta, y_w} = \frac{1}{|y_w|} \nabla_\theta \log \pi_\theta(y_w \mid x)$ and $\nabla_{\theta, y_l} = \frac{1}{|y_l|} \nabla_\theta \log \pi_\theta(y_l \mid x)$ correspond to the gradients that increase the probability of the ``winning'' response $y_w$ and decrease the probability of the ``losing'' response $y_l$, respectively.
The scaling factor $\beta$ in Equation \ref{eq:simpo_grad} linearly amplifies gradient magnitudes but does not alter the relative update directions in adaptive optimizers like Adam \cite{kingma2014adam}, as the momentum terms automatically normalize scale variations. We therefore omit $\beta$ in Figure \ref{fig:intro_gradient} for clearer visualization of the weighting function shapes.

The key insight is that \method{}'s ReLU-based gradient (Equation \ref{eq:repo_grad}) applies uniform updates only to pairs with $M_\theta < \gamma$, while SimPO's gradient (Equation \ref{eq:simpo_grad}) uses continuous $\beta$-scaled weights. Figure \ref{fig:intro_gradient} visualizes this difference, showing \method{} as the limiting case of SimPO as $\beta \to \infty$.

\begin{restatable}[Gradient Equivalence in the SimPO-to-\method{} Limit]{lemma}{gradientlimit}
\label{lemma:gradient_limit}
Under the same $M_\theta$ and $\gamma$ definitions, the SimPO gradient converges pointwise to the \method{} gradient as $\beta \to \infty$:
\begin{equation}
    \lim_{\beta \to \infty} \nabla_\theta \mathcal{L}_{\text{SimPO}} = \nabla_\theta \mathcal{L}_{\text{\method}}.
\end{equation}
\end{restatable}
\vspace{-0.5cm}
\begin{proof}[Sketch]
The convergence follows from the pointwise limit of the sigmoid weighting:  
\[
\lim_{\beta \to \infty} s_\theta = \lim_{\beta \to \infty} \sigma(\beta(-M_\theta + \gamma)) = \mathbb{I}(M_\theta < \gamma).
\]
Substituting this into Equation \ref{eq:simpo_grad} yields Equation~\ref{eq:repo_grad}.
\end{proof}

\begin{remark}
Please check Appendix for all proofs.
Lemma \ref{lemma:gradient_limit} establishes \method{} as the asymptotic limit of SimPO with large $\beta$, explaining two key advantages we will demonstrate in Section \ref{sec:performance_comparison}: comparable performance without $\beta$ tuning complexity, and an effective binary thresholding mechanism that induces implicit data filtering for controlling over-optimization.
\end{remark}
\subsection{Empirical Study}
\label{sec:performance_comparison}

The previous section analyzes the relationship between SimPO and \method{} from the perspective of gradient behavior. In this section, we compare their performance from an empirical standpoint.

\definecolor{linecolor}{RGB}{220,122,130}
\definecolor{barcolor}{RGB}{141,186,255}

\textbf{Experimental setup.}
We evaluate this approach using SimPO's experimental setup \cite{SimPO} with Llama3-8B and Gemma2-9B models (Instruct setup). For consistency, we use the same training datasets as SimPO: \href{https://huggingface.co/datasets/princeton-nlp/llama3-ultrafeedback-armorm}{princeton-nlp/llama3-ultrafeedback-armorm} for Llama3-8B and \href{https://huggingface.co/datasets/princeton-nlp/gemma2-ultrafeedback-armorm}{princeton-nlp/gemma2-ultrafeedback-armorm} for Gemma2-9B. For all SimPO experiments, we set $\beta=10.0$ and $\gamma=0.4$ for Gemma2-9B and $\beta=10.0$ and $\gamma=0.3$ for Llama3-8B, unless otherwise specified. 
We track optimization progress using two reward margin metrics:
\begin{equation}
    m_{\text{batch}} = \mathbb{E}_{(x,y_w,y_l) \in \mathcal{B}} [M_\theta], \quad
    m_{\mathcal{D}} = \mathbb{E}_{(x,y_w,y_l) \in \mathcal{D}} [M_\theta],
    \label{eq:margins}
\end{equation}
measuring response separation within each batch ($m_{\text{batch}}$) and across the entire training set ($m_{\mathcal{D}}$).
\begin{wrapfigure}{r}{0.6\textwidth}
    \centering
    \vspace{-0.55cm}
    \!\!\!\!\!\!\!\! \includegraphics[width=0.6\textwidth]{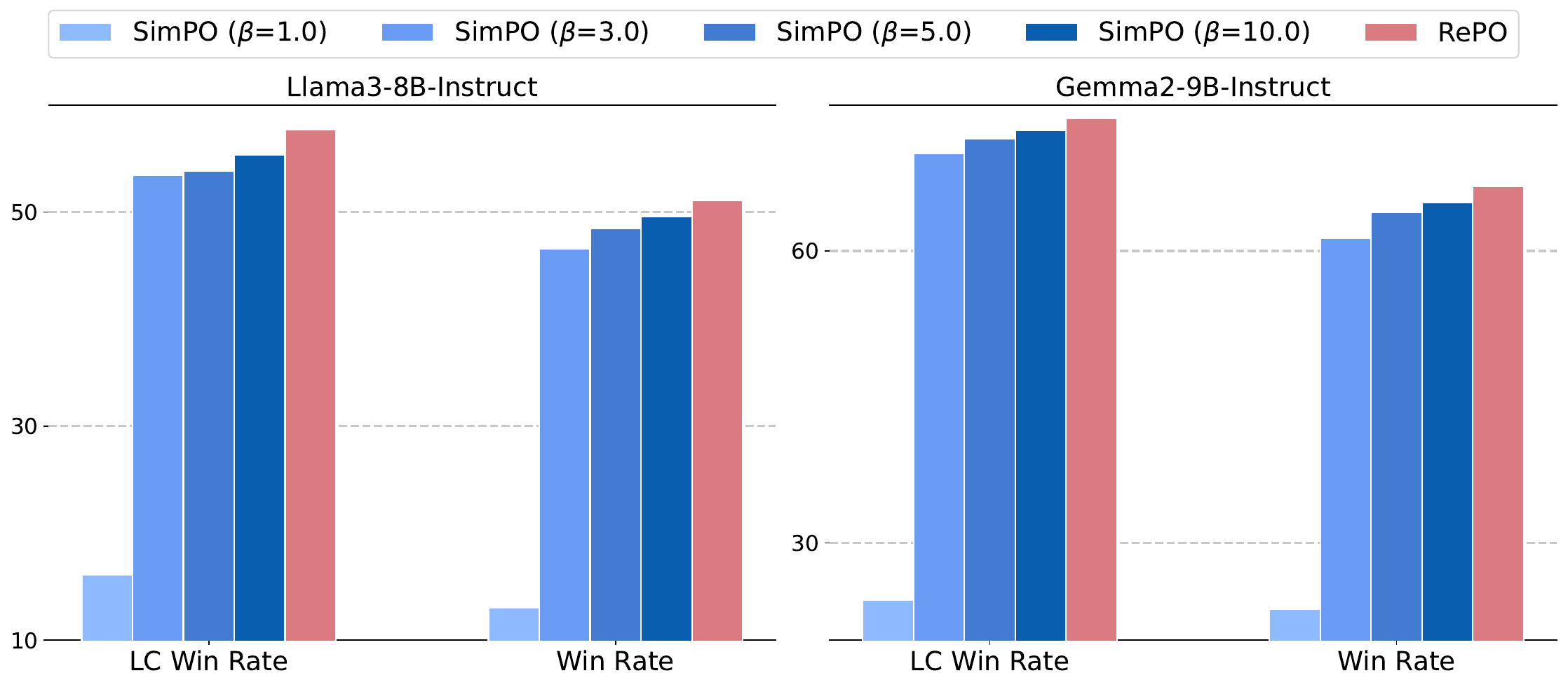}
    \vspace{-0.3cm}
    \caption{
    Performance of SimPO with varying $\beta$ and RePO on AlpacaEval2 benchmark.
    }
    \vspace{-0.8cm}
    \label{fig:method_varying_beta}
\end{wrapfigure}
\textbf{Evaluation benchmarks.}
We evaluate on two established benchmarks for open-ended generation: AlpacaEval 2 \citep{AlpacaEval} (measuring instruction-following quality against GPT-4) and Arena-Hard \citep{arenahard2024} (testing complex reasoning). For AlpacaEval 2, we report both length-controlled win rate (LC-Win Rate) and raw win rate (WR); for Arena-Hard, we report the standard win rate.

\textbf{Observation 1: Large \(\beta\) enhances SimPO's performance when paired with appropriate \(\gamma\).} 
We systematically evaluate SimPO across varying values of $\beta \in \{1.0, 3.0, 5.0, 10.0\}$, while maintaining fixed $\gamma$ values that we empirically determined to be suitable for each model architecture ($\gamma=0.4$ for Gemma2-9B and $\gamma=0.3$ for Llama3-8B). As shown in Figure \ref{fig:method_varying_beta}, increasing $\beta$ leads to consistent performance improvements across all evaluation metrics, with diminishing returns observed beyond $\beta=5.0$. These findings align with observations in the SimPO paper\footnote{In their \href{https://github.com/princeton-nlp/SimPO}{official repository}, the authors note: \textit{``SimPO requires a much larger $\beta$ than DPO... In many cases, an even larger (\eg 10) could yield better results.''}}.

\definecolor{mygreen}{RGB}{93,143,72}
\definecolor{mygray}{RGB}{178,178,178}
\begin{figure*}[t]
    \includegraphics[width=\columnwidth]{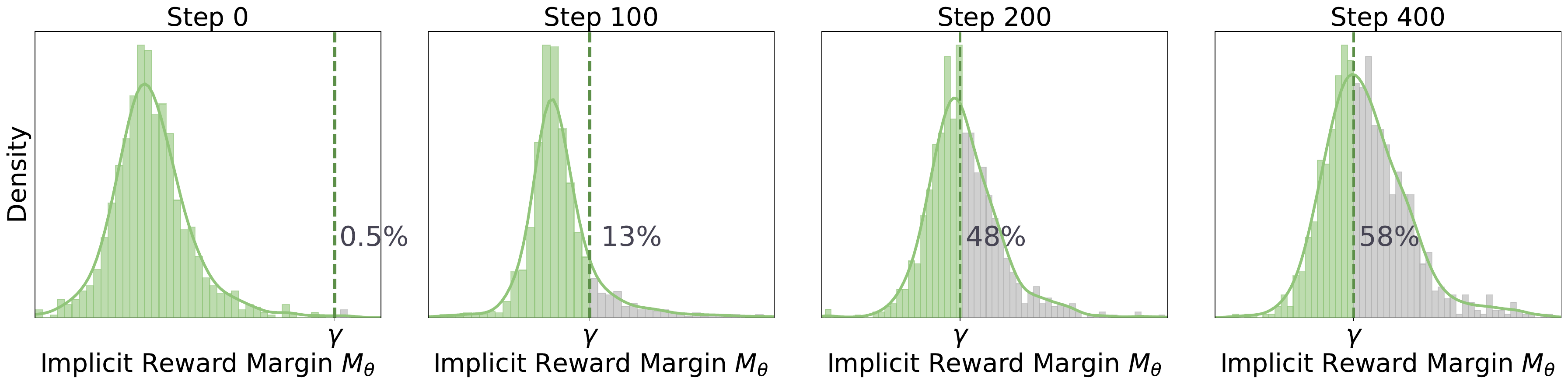}
    \caption{
Implicit reward margin \( M_\theta \) distribution across training steps (total: 467) for \method{} at \(\gamma = 0.4\). Dashed line: \(\gamma = 0.4\). \textcolor{mygreen}{Green}: samples below \(\gamma\) (gradient descent); \textcolor{mygray}{gray}: samples above \(\gamma\) (zero gradient). Numbers: fraction of samples above \(\gamma\).
    }
    \label{fig:step_filter}
    \vspace{-15pt}
\end{figure*}

\begin{wrapfigure}{r}{0.65\textwidth}
    \centering
    \vspace{-0.55cm}
    \!\!\!\!\!\!\!\! \includegraphics[width=0.65\textwidth]{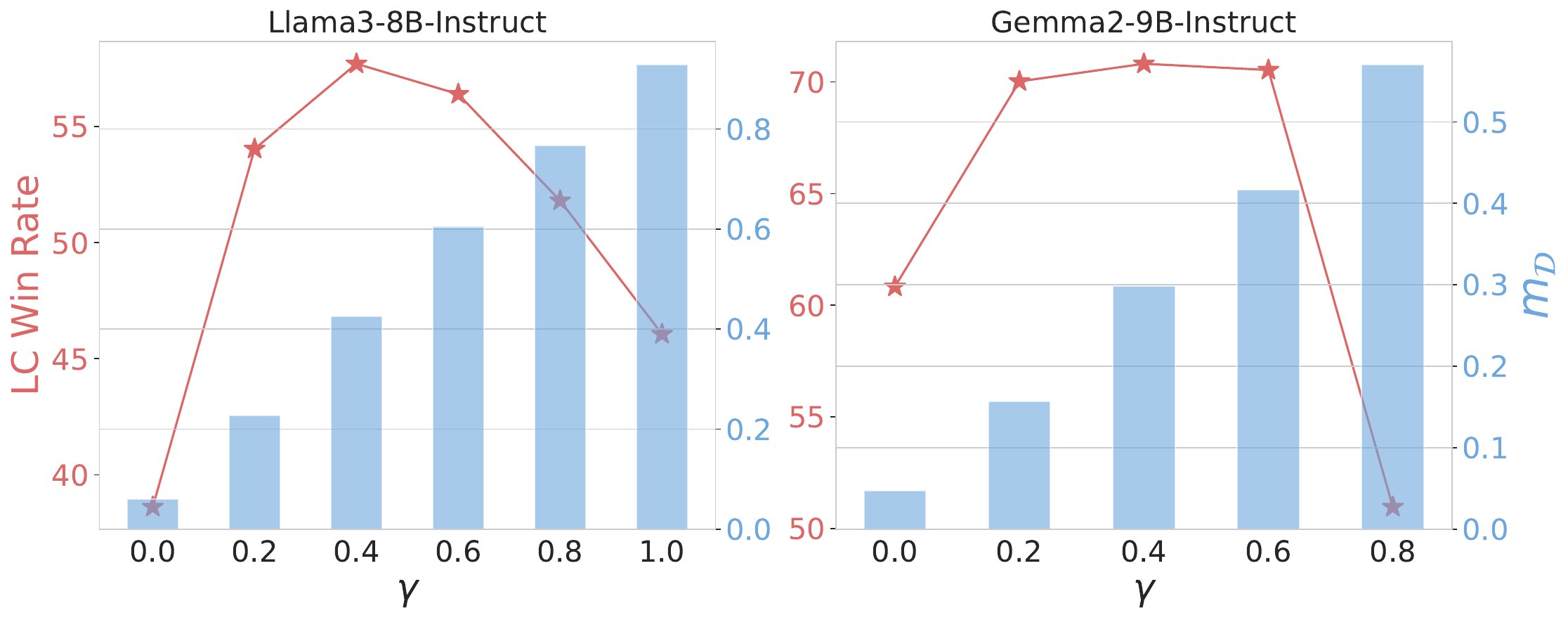}
    \vspace{-0.3cm}
    \caption{
        \textcolor{linecolor}{Line plot} of \method{} performance (AlpacaEval 2 LC Win Rate) and \textcolor{barcolor}{bar chart} of mean reward margins ($m_\mathcal{D}$) across varying $\gamma$ values.
        See Appendix \ref{sec_appendix_vary_gamma} for details.
    }
    \vspace{-0.5cm}
    \label{fig:varying_gamma}
\end{wrapfigure}
\textbf{Observation 2: \method matches high-$\beta$ SimPO.}
\method{} achieves performance comparable to SimPO with a large $\beta$. As shown in Figure \ref{fig:method_varying_beta}, \method{} achieves win rates of 51.1\% on Llama3-8B and 66.6\% on Gemma2-9B, comparable to SimPO's performance. 
This aligns with Lemma \ref{lemma:gradient_limit}, which establishes that \method{} can be interpreted as a limiting case of SimPO as $\beta \to \infty$.

\textbf{Observation 3: \(\gamma\) threshold creates a natural alignment-optimization tradeoff.}
Our experiments track the mean implicit reward margin $m_{\mathcal{D}}$ (\cf Equation \ref{eq:margins}) across training pairs. As Figure \ref{fig:varying_gamma} illustrates, increasing $\gamma$ directly elevates $m_{\mathcal{D}}$ while performance follows an inverted U-shaped pattern --- improving initially but declining beyond a critical threshold. In \method{}, gradients vanish when the implicit reward margin exceeds $\gamma$, effectively filtering out well-separated pairs from updates. This mechanism creates a fundamental tradeoff: small $\gamma$ values retain excessive zero-gradient samples causing under-filtering, while large $\gamma$ values force updates on most samples, potentially leading to over-optimization \citep{AmodeiOSCSM16} and ultimately degrading performance.

\textbf{Observation 4: RePO creates a natural learning curriculum via progressive filtering.}  
Figure \ref{fig:step_filter} reveals an unexpected pattern in how the distribution of implicit reward margins $M_\theta$ evolves throughout training. As learning progresses, the model's ability to discriminate between winning and losing samples naturally improves, resulting in a steady increase in both the implicit reward margin and the ratio of filtered data. Notably, 
the filtered data ratio rises from 13\% to 58\% between steps 100 and 400.
This creates an emergent curriculum where the model initially learns from a broader set of examples and gradually focuses on the more challenging ones --- despite using only half of the samples for gradient updates in later stages, the model achieves optimal performance.

\subsection{Over-Optimization Analysis}
The study of over-optimization can be traced back to traditional RLHF literature, and has been empirically investigated in both controlled experiments \cite{GaoSH23}  and user studies \cite{DuboisLTZGBGLH23}. In this work, we follow their experimental setup to further explore this phenomenon.

\textbf{Model Over-Optimization:} 
Building on \citet{DAAscalinglaw}, we investigate over-optimization in \method{}, by evaluating six different values of $\gamma$ ($0.0, 0.2, 0.4, 0.6, 0.8, 1.0$), each corresponding to varying levels of data filtering. Across all cases, we observe a distinct hump-shaped performance pattern: while moderate filtering improves alignment, excessive filtering causes performance to degrade, highlighting the over-optimization effect.

\begin{wrapfigure}{r}{0.6\textwidth}
    \centering
    \vspace{-0.25cm}
    \!\!\!\!\!\!\!\! \includegraphics[width=0.6\textwidth]{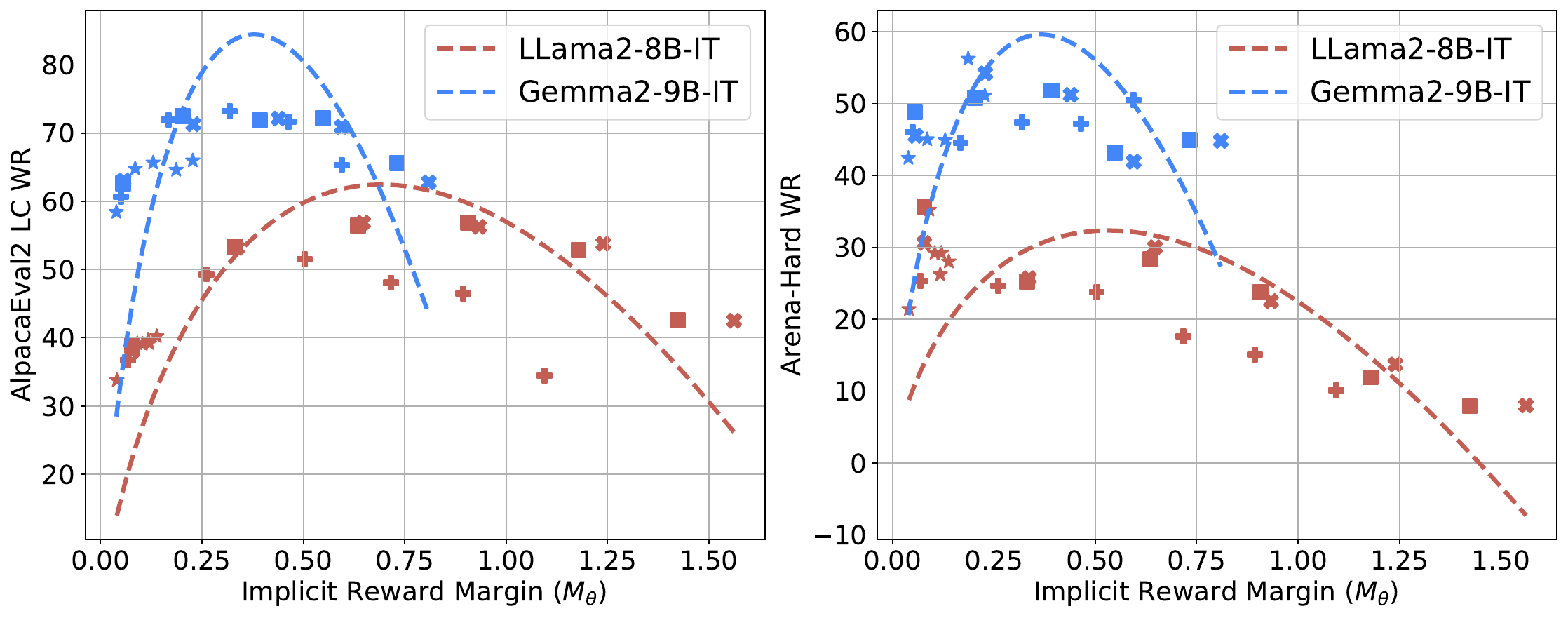}
    \vspace{-0.3cm}
    \caption{
        Over-optimization patterns for \method{} on Llama3-8B-IT and Gemma2-9B-IT, using AlpacaEval2 LC win rates and Arena-Hard raw win rates. Dotted curves represent theoretical fits based on \citet{GaoSH23}'s scaling laws, using GPT-4 win rates instead of standard reward model scores.
    }
    \vspace{-0.1cm}
    \label{fig:scaling_law}
  \end{wrapfigure}
\textbf{Scaling Law Fits.} 
Previous work \citep{GaoSH23, DAAscalinglaw} has established scaling laws for reward model scores as a function of the KL divergence between the initial and optimized policies. In contrast, we eliminate the reference model and the associated computational cost of calculating KL divergence. Instead, we use the mean implicit reward margin during training as a proxy metric. The reward function $R(d)$ is given by:
\begin{equation}
    R(d) = d (\alpha - \beta \log d),
\end{equation}

where $\alpha$ and $\beta$ are constants dependent on the reward model's dataset size and parameter count, and $d = m_{\text{batch}}$. Without training a proxy reward model, we substitute GPT-4 win rates over dataset completions for the gold reward. Interestingly, we find that this scaling law provides an accurate relationship between $d$ and win rates for \method{}.

\subsection{\newmethod: Exploring Extensions of ReLU-based Filtering}
\label{sec:method++}
While exploring ReLU's thresholding behavior, we observed an interesting limitation: for cases where the implicit reward margin is smaller than $\gamma$, their gradient weights become uniform, not differentiating between samples of varying difficulty.

This observation naturally led us to wonder: could we preserve the effective filtering mechanism while reintroducing some degree of weighting? To explore this question, we experimented with combining ReLU's binary filtering with SimPO's continuous weighting:
\begin{equation}
    \begin{aligned}
        \mathcal{L}_{\text{\newmethod}}(\pi_\theta) = - \mathbb{E}_{\mathcal{D}} \left[ \log \sigma  \left( - \text{ReLU} \left( -\beta \left(M_\theta - \gamma\right) \right) \right) \right],
        \label{eq:new_method}
    \end{aligned}
\end{equation}

This exploration was a natural follow-up to our main discovery about ReLU's effectiveness, rather than our primary contribution. We were curious to see whether combining the best aspects of both approaches might yield additional insights about preference learning mechanisms.

\textbf{What does this combined approach reveal?} To understand the behavior of this extension, we examined its gradient with respect to the parameters $\theta$:
\begin{equation}
    \begin{aligned}
        \nabla_\theta\mathcal{L}_{\mathrm{\newmethod}}(\pi_\theta)=-\beta\mathbb{E}_{\mathcal{D}}\left[s_\theta\cdot \mathbb{I}_\theta \cdot  \left( \nabla_{\theta,y_w} - \nabla_{\theta,y_l} \right)\right],
    \end{aligned}
\end{equation}
where $s_\theta=\sigma\left(\beta \left( -M_\theta +\gamma\right)\right)$ and $\mathbb{I}_\theta$ is an indicator function that is 1 if $M_\theta<\gamma$ and 0 otherwise. 

We observed that this gradient combines properties we discovered in both approaches: it scales updates by 
$s_{\theta}$ (similar to SimPO) and filters them using $\mathbb{I}_\theta$ (the key discovery in our ReLU exploration), focusing the model on less-separated pairs while giving higher weights to smaller separations.

\textbf{Adaptation with \newmethod{}}. The core contribution of \newmethod{} lies in leveraging ReLU to mitigate over-optimization while preserving the standard workflow of preference optimization. This makes \newmethod{} easily adaptable to existing DPO-like methods. For instance, as shown in Equation \ref{eq:new_method}, replacing $M_\theta$ with $\log \frac{\pi_\theta(y_w \mid x)}{\pi_\theta(y_l \mid x)} - \log \frac{\pi_{\text{ref}}(y_w \mid x)}{\pi_{\text{ref}}(y_l \mid x)}$ seamlessly integrates \newmethod{} into DPO, forming a ReLU-enhanced version of DPO.

\section{Theoretical Analysis: ReLU's Optimality in Preference Learning}
Next, we establish a surprising theoretical connection between preference optimization and binary classification, revealing why our simple ReLU-based approach achieves superior performance.

Following \citet{GPO}, preference learning can be reformulated as binary classification. Given pairs $(z, l)$ where $z \in \mathbb{R}^k$ and $l \in \{-1, 1\}$, we aim to learn a predictor $\hat{\ell}(z)$ whose sign matches $l$. The classification accuracy is:
\(
\frac{1}{2} \mathbb{E}[\mathrm{sign}(\hat{\ell}(z) \cdot l)] + \frac{1}{2}.
\)
This corresponds to minimizing the 0-1 loss:
\begin{equation}
\mathcal{L}_{0\text{-}1}(\hat{\ell}) := \mathbb{E}\left[1 - \mathrm{sign}(\hat{\ell}(z) \cdot l)\right]
\end{equation}
For preference data $(y_w, y_l)$ where $y_w \succ y_l$, we set $l = 1$ and parameterize $\hat{\ell}(y_w, y_l) = r_\phi(y_w) - r_\phi(y_l)$, yielding the objective:
\begin{equation}
\mathcal{L}_f(\hat{\ell}) := \mathbb{E}\left[f(r_\phi(y_w) - r_\phi(y_l))\right]
\end{equation}

Where $f$ determines the surrogate loss: $f(x) = \mathbb{I}(x < 0)$ gives the 0-1 loss, $f(x) = -\log \sigma(x)$ yields SimPO's logistic loss, and $f(x) = \text{ReLU}(-x)$ gives our method's loss.

Our key insight comes from analyzing the convex envelope of the 0-1 loss:

\begin{definition}
The convex envelope of $\mathcal{L}_{0\text{-}1}$ over a closed convex set $D \subseteq \mathbb{R}$ is:
\begin{equation}
\mathrm{conv}_D \mathcal{L}_{0\text{-}1}(x) := \sup \left\{ h(x) \mid h \text{ is convex}, \, h \leq \mathcal{L}_{0\text{-}1} \, \forall x \in D \right\}
\end{equation}
\end{definition}

\begin{restatable}[ReLU as Convex Envelope]{theorem}{convexenv}
\label{thm:convex_env}
For $D = [-a, b]$ with $a, b > 0$, the convex envelope of $\mathcal{L}_{0\text{-}1}(x) = \mathbb{I}(x < 0)$ is:
\begin{equation}
\mathrm{conv}_D \mathcal{L}_{0\text{-}1}(x) = \frac{1}{a} \mathrm{ReLU}(-x)
\end{equation}
\end{restatable}

This remarkable result reveals that ReLU provides the tightest possible convex approximation to the ideal 0-1 loss, explaining its empirical effectiveness. Furthermore:

\begin{restatable}[Optimality Preservation]{corollary}{optpreserve}
\label{cor:opt_preserve}
Let $D \subseteq \mathbb{R}$ be convex. Then:
\begin{equation}
\arg\min_{\hat{\ell}} \mathcal{L}_{0\text{-}1}(\hat{\ell}) = \arg\min_{\hat{\ell}} \mathrm{conv}_D \mathcal{L}_{0\text{-}1}(\hat{\ell})
\end{equation}
And for $D = [-a, b]$:
\begin{equation}
\arg\min_{x \in D} \mathcal{L}_{0\text{-}1}(x) = \arg\min_{x \in D} \frac{1}{a} \mathrm{ReLU}(-x)
\end{equation}
\end{restatable}

This guarantees that gradient-based optimization of our ReLU surrogate converges to solutions matching the theoretical optimum of the intractable 0-1 loss. Importantly:

\begin{restatable}[Logistic Loss Suboptimality]{corollary}{nonconvex}
\label{cor:non_convex}
The logistic loss $f_{\mathrm{log-sigmoid}}(x) = -\log \sigma(x)$ is not the convex envelope of $\mathcal{L}_{0\text{-}1}$.
\end{restatable}

This theoretical foundation explains why our simple ReLU-based approach consistently outperforms more complex mechanisms like SimPO's sigmoid weighting --- ReLU provides optimality guarantees that logistic loss cannot match, while being computationally more efficient.

\section{Experiments}
\label{sec:experiments}
In this section, we examine how our simplified ReLU-based approach behaves across different models and settings. Rather than focusing solely on performance gains, we explore patterns that help explain why such a simple mechanism works effectively in practice.

\subsection{Experimental Setup}
The core experimental configuration extends our investigation from Section \ref{sec:performance_comparison} to include Mistral2-7B \citep{Jiang2023Mistral7} alongside previously examined models. For the Llama3-Instruct v0.2 experiments, we employed the \href{https://huggingface.co/RLHFlow/ArmoRM-Llama3-8B-v0.1}{RLHFlow/ArmoRM-Llama3-8B-v0.1} \citep{ArmoRMLlama3} reward model for ranking generated data.
We benchmark our approach against established preference optimization methods: DPO \citep{DPO}, SimPO \citep{SimPO}, IPO \citep{Azar2023AGT}, CPO \citep{xu2024contrastive}, KTO \citep{Ethayarajh2024KTOMA}, ORPO \citep{Hong2024ORPOMP}, and R-DPO \citep{Park2024DisentanglingLF}, with SFT models serving as baselines. Implementation details are provided in Appendix \ref{sec:appendix_implement}.
We also evaluate on downstream tasks from the Huggingface Open Leaderboard benchmarks \cite{open-llm-leaderboard}, with additional details in in Appendix \ref{sec:downstram_task}.
The code is available at \url{https://github.com/junkangwu/RePO}.

\begin{table*}[t]
    \caption{\textbf{AlpacaEval 2 (AE2), Arena-Hard (AH) results across four settings.} ``WR'' denotes the raw win rate,``LC'' the length-controlled win rate. The best results are highlighted in bold, while the second-best are underlined.}
    \vspace{-15pt}
    \label{tab:main_results}
    \begin{center}
    \resizebox{1.0\textwidth}{!}{%
    \begin{tabular}{lcccccccccccc}
    \toprule
    \multirow{3}{*}{\textbf{Method}} & \multicolumn{3}{c}{\textbf{Llama3-Instruct (8B)}} & \multicolumn{3}{c}{\textbf{Mistral-Instruct (7B)}} &
     \multicolumn{3}{c}{\textbf{Llama3-Instruct v0.2 (8B)}} & \multicolumn{3}{c}{\textbf{Gemma2-Instruct (9B)}}\\
    \cmidrule(lr){2-4} \cmidrule(lr){5-7} \cmidrule(lr){8-10} \cmidrule(lr){11-13}
    & \multicolumn{2}{c}{\textbf{AE 2}} & \multicolumn{1}{c}{\textbf{AH}} & \multicolumn{2}{c}{\textbf{AE 2}} & \multicolumn{1}{c}{\textbf{AH}} &
    \multicolumn{2}{c}{\textbf{AE 2}} & \multicolumn{1}{c}{\textbf{AH}} & \multicolumn{2}{c}{\textbf{AE 2}} & \multicolumn{1}{c}{\textbf{AH}}
    \\
    \cmidrule(lr){2-3} \cmidrule(lr){4-4} \cmidrule(lr){5-6} \cmidrule(lr){7-7} \cmidrule{8-9} \cmidrule{10-10} \cmidrule{11-12} \cmidrule{13-13}
        & \textbf{LC} & \textbf{WR} &  \textbf{WR}
    &  \textbf{LC} & \textbf{WR} &  \textbf{WR}
    &  \textbf{LC} & \textbf{WR} &  \textbf{WR}
    &  \textbf{LC} & \textbf{WR} &  \textbf{WR} \\
    \midrule
    SFT & 24.0  & 23.6 & 22.4 &  19.0 & 15.4  &  12.9 
    & 24.0  & 23.6  & 22.4  &  48.7 & 36.5  & 42.1 \\ \midrule
    SLiC-HF & 26.9 & 27.5  & 26.2 & 24.1 & 24.6  & 18.9  
    & 33.9 & 32.5  & 29.3 &  65.1 & 60.5  & 53.7 \\
    DPO & 40.2 & \underline{38.1}  & 31.2 & 20.3 & 17.9  & 13.4  
    & 48.2 & 47.5  & \textbf{35.2} & 70.4 & \textbf{66.9} & \underline{58.8} \\
    IPO & 35.9 & 34.4  & 30.2 & 22.3 & 18.6 &  16.2  
    & 40.6 & 39.6  & 34.9 & 62.6 & 58.4 & 53.5 \\
    CPO & 29.6 & 34.4  & 29.4 & 26.2 & 31.7  & \textbf{23.8}  
    & 36.5 & 40.8  & 34.2 & 56.4 & 53.4 & 55.2  \\
    KTO & 38.3 & 34.1  & 30.3 & 19.4 & 20.3 & 16.8   
    & 41.4 & 36.4 & 28.9 & 61.7 & 55.5 & 53.8  \\
    ORPO& 31.6 & 29.8  & 26.3 & 24.0 & 23.0 & 18.6    
    & 36.5 & 33.1 & 30.4 & 56.2 & 46.7 & 46.2 \\
    R-DPO & 40.3 & 37.3  & \underline{32.9} & 21.4 & 22.2 & 13.8   
    & 51.6 & \underline{50.7}  & \underline{35.0} & 68.3 & \textbf{66.9} & {57.9} \\
    SimPO & \underline{43.8} & {38.0}  & 32.6 & \underline{30.2} & \underline{32.1} & 20.1  
    & \underline{55.6} & 49.6  & 33.6 & \underline{72.4} & 65.0 & 57.8 \\
     \midrule
    \method{} & \textbf{46.7} & \textbf{41.1}  & \textbf{33.3} & \textbf{30.4} & \textbf{33.6}  & \underline{20.3} 
    & \textbf{57.7} & \textbf{51.1}  & \textbf{35.2} & \textbf{73.6} & \underline{66.6}  & \textbf{59.1}  \\
    \bottomrule
    \end{tabular}
    }
    \end{center}
    \vspace{-15pt}
\end{table*}
\definecolor{-}{rgb}{0.25,0.41,0.88}
\definecolor{+}{rgb}{0.70,0.13,0.13}

\subsection{Result Comparisons}
\textbf{Observation: Simple ReLU thresholding exhibits surprising effectiveness.} 
Table~\ref{tab:main_results} reveals an unexpected pattern: despite removing components previously thought essential, the simple ReLU-based approach consistently performs well across all evaluated models and benchmarks. This finding aligns with our theoretical analysis showing that binary thresholding directly approximates the convex envelope of the 0-1 loss. On AlpacaEval 2, we observe improvements of 0.2-2.8 points in LC win rates across different configurations compared to the strongest baselines.

\subsection{Methodology Comparisons}
Beyond alignment, we also compare the methodologies of these preference learning methods. Our method plays a hub to connect these methods. 

\textbf{Relation to SimPO.}
SimPO employs sigmoid weighting via log-sigmoid activation to attenuate gradients as reward margins increase, mitigating over-optimization. RePO can be viewed as ``SimPO without log-sigmoid,'' replacing this continuous scaling with binary filtering.
To validate this relationship, we integrated a ReLU-based filtering mechanism into SimPO (\cf \newmethod{} Equation \ref{eq:new_method}). Table \ref{tab:ablation} confirms that ReLU's filtering mechanism enhances performance. 
As demonstrated in Section \ref{sec:method++}, \newmethod{} directly addresses over-optimization while retaining the benefits of some weighting.

\textbf{Relation to SLiC-HF.}
RePO can be characterized as ``SLiC-HF without SFT regularization''. To ensure a fair comparison (while disregarding differences in length normalization), we investigated the impact of SFT regularization by varying its coefficient, $\lambda$. 
The results, presented in Appendix Table \ref{tab:hyperparams_slichf} and further details in Appendix \ref{sec:appendix_slichf_repo},
indicate that this additional regularization term offers no discernible improvement. This suggests that SFT regularization targets a different optimization challenge, distinct from the direct over-optimization problem RePO addresses.

\textbf{Relation to DPO.}
Mathematically, DPO is equivalent to SimPO when the margin $\gamma$ is defined as $\log\pi_{\text{ref}}(y_w \mid x) - \log\pi_{\text{ref}}(y_l \mid x)$ (ignoring length normalization). However, directly substituting the log-sigmoid function with ReLU in DPO's formulation leads to a significant performance degradation (see Appendix Table \ref{tab:ablation2} and Appendix \ref{sec:appendix_dpo_repo}). This underscores the critical role of the threshold $\gamma$ in determining the effectiveness of over-optimization prevention. As identified by \citet{AlphaDPO}, reference model based reward margins are often unreliable as target margins, which explains why SimPO's explicit $\gamma$ parameter is effective for preference learning.

\subsection{Effect of ReLU Filtering Across Methods}
Having observed the effectiveness of binary thresholding, we naturally questioned whether this mechanism might enhance other preference learning approaches. Table~\ref{tab:ablation} shows that integrating ReLU filtering consistently improved performance across both DPO and SimPO frameworks, suggesting that selective gradient application based on margin thresholds provides benefits beyond our specific implementation.
Our experiments with the combined approach (\cf \newmethod{} in Section \ref{sec:method++}) revealed particularly strong improvements when applied to DPO (5\%–12\% gains), with notable performance on Arena-Hard (reaching 65.7). 
This design effectively mitigates over-optimization while preserving the benefits of the original scheme. 

\begin{wraptable}{r}{0.6\textwidth}  
    \vspace{-12pt}
    \caption{
    \textbf{Performance improvements of \method{} and \newmethod{} over DPO and SimPO.} Results are present on AlpacaEval 2 (AE 2) and Arena-Hard (AH) with LC (\%) and WR (\%). Red numbers indicate relative improvements.}
    \label{tab:ablation}
    \vspace{-10pt}
    \begin{center}
    \resizebox{\linewidth}{!}{%
    \begin{tabular}{lllllll}
    \toprule
    \multirow{3}{*}{\textbf{Method}} & \multicolumn{3}{c}{\textbf{Llama3-Instruct v0.2 (8B)}} & \multicolumn{3}{c}{\textbf{Gemma2-Instruct (9B)}} \\
    \cmidrule(lr){2-4} \cmidrule(lr){5-7}
    & \multicolumn{2}{c}{\textbf{AE 2}} & \multicolumn{1}{c}{\textbf{AH}} & \multicolumn{2}{c}{\textbf{AE 2}} & \multicolumn{1}{c}{\textbf{AH}} \\
    \cmidrule(lr){2-3} \cmidrule(lr){4-4} \cmidrule(lr){5-6} \cmidrule(lr){7-7}
    & \textbf{LC} & \textbf{WR} &  \textbf{WR}
    &  \textbf{LC} & \textbf{WR} &  \textbf{WR} \\
    \midrule
    DPO & 48.2 & 47.5  & 35.2 & 70.4 & {66.9} & {58.8}  \\
    \textit{w.} \method{} & $50.3^{\color{+}+4.4\%}$  & $51.8^{\color{+}+9.1\%}$ & $38.2^{\color{+}+8.5\%}$ &  $73.8^{\color{+}+4.8\%}$ & $71.0^{\color{+}+6.1\%}$  & $64.2^{\color{+}+9.2\%}$ \\ 
    \textit{w.} \method{}++ & $50.8^{\color{+}+5.4\%}$ & $52.2^{\color{+}+9.9\%}$  & $37.2^{\color{+}+5.7\%}$ &  $71.8^{\color{+}+2.0\%}$ & $69.5^{\color{+}+3.9\%}$  &  $65.7^{\color{+}+11.7\%}$  \\ 
    \midrule
    SimPO & $55.6$ & $49.6$  & $33.6$ & $72.4$ & $65.0$ & $57.8$  \\
    \textit{w.} \method{} & $57.7^{\color{+}+3.8\%}$ & $51.1^{\color{+}+3.0\%}$  & $35.2^{\color{+}+4.8\%}$ & $73.6^{\color{+}+1.7\%}$ & $66.6^{\color{+}+2.5\%}$ & $59.1^{\color{+}+2.2\%}$ \\
    \textit{w.} \method{}++ & $56.1^{\color{+}+0.9\%}$ & $50.1^{\color{+}+1.0\%}$  & $35.9^{\color{+}+6.8\%}$ & $74.1^{\color{+}+2.3\%}$ & $66.5^{\color{+}+2.3\%}$ & $59.8^{\color{+}+3.5\%}$ \\
    \bottomrule
    \end{tabular}
    }
    \vspace{-15pt}
    \end{center}
\end{wraptable}
\subsection{Dynamic Margin Scheduling and Curriculum Learning}
Our investigation into the role of target reward margin $\gamma$ led us to an unexpected discovery about curriculum learning. We experimented with dynamic scheduling of $\gamma$ throughout training, implementing two strategies: (i) increasing $\gamma$ from small to large, and (ii) decreasing $\gamma$ from large to small.
Figure~\ref{fig:dynamic_gamma} reveals a striking pattern: starting with a moderately large value of $\gamma$ and gradually decreasing it ($1.0 \rightarrow 0.2$) naturally creates an effective curriculum that improves model performance. In contrast, both excessively large values ($1.0 \rightarrow 0.8$) and small values ($0.0 \rightarrow 0.2$) led to suboptimal outcomes.

\begin{figure}[t]
    \includegraphics[width=\columnwidth]{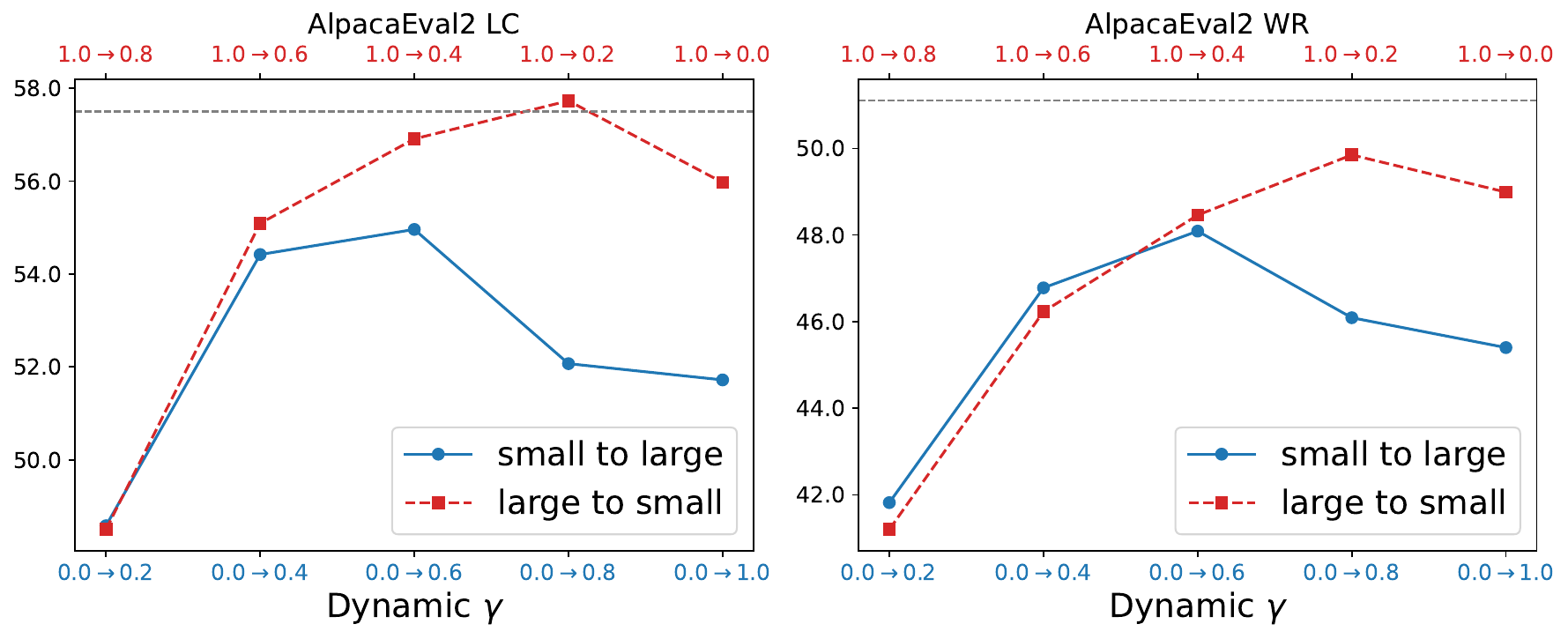}
    \vspace{-15pt}
    \caption{
    Exploration of dynamic $\gamma$ scheduling on Llama3-Instruct v0.2 (8B). The dashed line represents performance with a fixed $\gamma$. We observed that decreasing $\gamma$ from an initially larger value creates a natural curriculum that enhances performance. 
    }
    \label{fig:dynamic_gamma}
    \vspace{-15pt}
\end{figure}

This observation reveals an intriguing self-regulating property: early in training when the model is underfitting, a larger $\gamma$ permits more aggressive updates across more examples. As training progresses, the decreasing $\gamma$ naturally focuses learning on increasingly challenging examples, effectively preventing over-optimization. This emergent curriculum behavior, arising from a simple parameter schedule, suggests that binary thresholding captures fundamental learning dynamics that more complex mechanisms might obscure.

\section{Discussion}

\textbf{Conclusion}
\label{sec:conclusion}
Our exploration of simple ReLU activation in preference learning has revealed several key insights. We found that binary thresholding, implemented through a straightforward ReLU function, provides an effective mechanism for preventing over-optimization in language model alignment. Our theoretical analysis showed that this seemingly simple approach is, in fact, the convex envelope of the ideal 0-1 loss function, explaining its surprising effectiveness. Rather than developing yet another complex preference optimization method, our work uncovered how fundamental properties like data selection and implicit curriculum learning emerge naturally from basic principles.

\textbf{Limitations and future directions.}
Our current exploration is limited to offline preference learning settings. Future work could investigate how these insights might extend to online learning scenarios, where preferences are gathered interactively. Additionally, while we found that a fixed margin threshold works well in practice, exploring adaptive or context-aware thresholds might further improve performance in highly dynamic environments. The relationship between binary filtering and self-play scenarios \citep{spin} --- where the model generates its own feedback --- is another promising direction that could lead to more scalable alignment techniques.

Beyond alignment, our work connects to LLM reasoning research \citep{o1, guo2025deepseek}. Future work should investigate how KL penalties and gradient clipping in GRPO \citep{shao2024deepseekmath} and PPO \citep{PPO} balance preventing over-optimization against preserving reasoning capabilities --- a critical consideration for advancing alignment methodologies.

\begin{ack}
This research was supported by the National Science and Technology Major Project (2023ZD0121102), the National Natural Science Foundation of China (U24B20180, 62302321), and the Fundamental Research Funds for the Central Universities (WK2100250065). This research also benefited from the advanced computing resources provided by the Supercomputing Center of the USTC.
\end{ack}

\bibliography{neurips_2025}
\bibliographystyle{unsrtnat}


\appendix
\section{Related Works}
\label{sec:related_work}
\textbf{Reinforcement learning from human feedback.}
RLHF is a technique designed to align large language models with human preferences and values~\citep{ChristianoLBMLA17, rlhf_2, InstructGPT, Azar2023AGT}. Traditional RLHF is typically structured in three stages: supervised fine-tuning~\citep{zhou2024lima, taori2023stanford, geng2023koala, DatabricksBlog2023DollyV2, kopf2024openassistant, Ding2023EnhancingCL}, reward modeling~\citep{GaoSH23, luo2023wizardmath, chen2024odin, lightman2023let, havrilla2024glore, lambert2024rewardbench}, and policy optimization~\citep{PPO, anthony2017thinking}. In the third stage, Proximal Policy Optimization (PPO) is widely adopted. In contrast, RLOO \cite{rloo} reduces the GPU memory footprint of RLHF by eliminating the Critic model and leverages a Leave-One-Out strategy to achieve superior performance. GRPO \cite{GRPO}, another variant of PPO, improves mathematical reasoning abilities while optimizing memory usage by replacing the Leave-One-Out method with a direct subtraction of the mean of all samples for a given prompt.

\textbf{Offline preference optimization.}
In addition to DPO, several alternative preference optimization objectives have been proposed. IPO~\citep{Azar2023AGT} addresses overfitting issues inherent in DPO. ORPO~\citep{Hong2024ORPOMP} and SimPO~\citep{SimPO} aim to remove reliance on a reference model. R-DPO~\citep{Park2024DisentanglingLF} targets the reduction of exploitation due to sequence length, while KTO~\citep{Ethayarajh2024KTOMA} handles preference optimization in the absence of pairwise data. CPO~\citep{xu2024contrastive} and $\beta$-DPO~\citep{margin2} focus on improving the quality of preference data.
Another research direction addresses noise in offline alignment, which arises from the need to construct pairwise data. rDPO~\citep{ChowdhuryKN24}, a variant of DPO, mitigates preference noise and enhances policy robustness, while DrDPO~\citep{DrDPO} applies distributed robust optimization to tackle this issue. Other works have approached the problem through divergence regularization~\citep{fpo, chi-Divergence}, selection of high-quality data~\citep{rs_dpo, eva}, or reweighting loss functions~\citep{reward_diff_dpo, active_pref_learn, a_recipe}.

\textbf{Iterative Preference Optimization.}
Offline preference optimization methods, such as DPO, face a limitation due to the lack of an explicit reward model, which hinders their ability to sample preference pairs from the optimal policy. To address this, iterative preference optimization techniques have been proposed. These methods iteratively update the reference model using the most recent policy model or generate new preference pairs in each iteration \cite{rlhf_workflow, sDPO, dno, xiong24icml, YuanPCLSXW24, spin, sppo, rebel}. For instance, SPIN \cite{spin} employs a self-play framework to fine-tune the model in a supervised manner, while \citet{YuanPCLSXW24} annotate preferences throughout the iterative process. REBEL improves sample quality by regressing the relative reward. Additionally, \cite{CalandrielloGMR24} generates data using a mixture policy, similar to the Nash-MD algorithm \cite{dno}.

\section{Broader Impacts}
\label{sec:broader}
This paper presents work whose goal is to advance the field of Machine Learning. There are many
potential societal consequences of our work, none of which we feel must be specifically highlighted
here
\section{Proofs}
\label{sec:proofs}
\subsection{Proof of Lemma \ref{lemma:gradient_limit}}
\gradientlimit*
\label{proof_gradient_limit}
\begin{proof}
We formally establish the gradient equivalence through pointwise convergence analysis. Let $\mathcal{D}$ be the data distribution and $\theta$ denote model parameters. Recall the gradient expressions:

\noindent\textbf{SimPO Gradient:}
\begin{equation}
    \nabla_\theta \mathcal{L}_{\mathrm{SimPO}} = -\beta \mathbb{E}_{\mathcal{D}} \left[ \sigma\big(\beta(-M_\theta + \gamma)\big) \cdot \left( \nabla_{\theta,y_w} - \nabla_{\theta,y_l} \right) \right]
\end{equation}

\noindent\textbf{\method{} Gradient:}
\begin{equation}
    \nabla_\theta \mathcal{L}_{\mathrm{\method}} = -\mathbb{E}_{\mathcal{D}} \left[ \mathbb{I}(M_\theta < \gamma) \cdot \left( \nabla_{\theta,y_w} - \nabla_{\theta,y_l} \right) \right]
\end{equation}

\noindent where $\sigma(\cdot)$ is the sigmoid function. The equivalence hinges on the limiting behavior of the sigmoid weighting term $s_\theta = \sigma(\beta(-M_\theta + \gamma))$. We analyze three cases:

\paragraph{Case 1: $M_\theta < \gamma$}  
Here, $-M_\theta + \gamma > 0$. As $\beta \to \infty$,  
\[
\lim_{\beta \to \infty} \sigma\big(\beta(-M_\theta + \gamma)\big) = \lim_{z \to \infty} \sigma(z) = 1 = \mathbb{I}(M_\theta < \gamma).
\]

\paragraph{Case 2: $M_\theta > \gamma$}  
Here, $-M_\theta + \gamma < 0$. As $\beta \to \infty$,  
\[
\lim_{\beta \to \infty} \sigma\big(\beta(-M_\theta + \gamma)\big) = \lim_{z \to -\infty} \sigma(z) = 0 = \mathbb{I}(M_\theta < \gamma).
\]

\paragraph{Case 3: $M_\theta = \gamma$}  
This occurs on a measure-zero set under continuous distributions. The limit becomes:  
\[
\lim_{\beta \to \infty} \sigma(0) = \frac{1}{2} \neq \mathbb{I}(M_\theta < \gamma),
\]
which is negligible in expectation.

Thus, $\lim_{\beta \to \infty} s_\theta = \mathbb{I}(M_\theta < \gamma)$ almost everywhere. Substituting this into the SimPO gradient:

\begin{align}
    \lim_{\beta \to \infty} \nabla_\theta \mathcal{L}_{\mathrm{SimPO}} 
    &= -\lim_{\beta \to \infty} \beta \mathbb{E}_{\mathcal{D}} \left[ s_\theta \cdot \left( \nabla_{\theta,y_w} - \nabla_{\theta,y_l} \right) \right] \\
    &= -\mathbb{E}_{\mathcal{D}} \left[ \lim_{\beta \to \infty} \beta s_\theta \cdot \left( \nabla_{\theta,y_w} - \nabla_{\theta,y_l} \right) \right] \quad \text{(Dominated Convergence Theorem)}
\end{align}

\noindent To resolve the $\beta$ scaling, observe that for $M_\theta \neq \gamma$:  
\begin{equation}
    \lim_{\beta \to \infty} \beta s_\theta = \begin{cases}
        \lim_{\beta \to \infty} \beta \cdot 1 = \infty & \text{if } M_\theta < \gamma \\
        \lim_{\beta \to \infty} \beta \cdot 0 = 0 & \text{if } M_\theta > \gamma
    \end{cases}
\end{equation}

\noindent The divergence when $M_\theta < \gamma$ is mitigated by adaptive optimizers like Adam, which normalize gradient magnitudes through momentum terms. Formally, let $g_\theta = \nabla_{\theta,y_w} - \nabla_{\theta,y_l}$. Under Adam's update rule:  
\[
\theta_{t+1} = \theta_t - \eta \cdot \frac{\hat{m}_t}{\sqrt{\hat{v}_t} + \epsilon},
\]
where $\hat{m}_t$ and $\hat{v}_t$ are bias-corrected momentum estimates. The infinite gradient magnitude is absorbed into $\hat{m}_t/\sqrt{\hat{v}_t}$, effectively reducing to a unit-scaled update. Hence, in normalized update space:  
\[
\lim_{\beta \to \infty} \beta s_\theta \cdot g_\theta \propto \mathbb{I}(M_\theta < \gamma) \cdot g_\theta.
\]

Combining these results:  
\begin{equation}
    \lim_{\beta \to \infty} \nabla_\theta \mathcal{L}_{\mathrm{SimPO}} = -\mathbb{E}_{\mathcal{D}} \left[ \mathbb{I}(M_\theta < \gamma) \cdot \left( \nabla_{\theta,y_w} - \nabla_{\theta,y_l} \right) \right] = \nabla_\theta \mathcal{L}_{\mathrm{\method}},
\end{equation}  
which completes the proof.
\end{proof}

\subsection{Proof of Theorem \ref{thm:convex_env}}
\convexenv*
\begin{proof}
We demonstrate that \( h(x) = \frac{1}{a}\mathrm{ReLU}(-x) \) satisfies the convex envelope definition through three sequential arguments.

\textbf{1. Convexity and Underestimation:}  
The ReLU function is convex as the pointwise maximum of affine functions (Rule 3). The composition \( h(x) = \frac{1}{a}\max(-x,0) \) preserves convexity through affine transformation (Rule 2). For all \( x \in D \):
\begin{itemize}
    \item When \( x < 0 \): \( h(x) = -\frac{x}{a} \leq 1 = \mathcal{L}_{0\text{-}1}(x) \), since \( x \geq -a \implies -\frac{x}{a} \leq 1 \)
    \item When \( x \geq 0 \): \( h(x) = 0 = \mathcal{L}_{0\text{-}1}(x) \)
\end{itemize}
Thus \( h(x) \leq \mathcal{L}_{0\text{-}1}(x) \) over \( D \).

\textbf{2. Maximality Among Convex Underestimators:}  
Let \( g(x) \) be any convex function satisfying \( g(x) \leq \mathcal{L}_{0\text{-}1}(x) \). For \( x \in [-a, 0) \), convexity implies:
\[
g(x) \leq \frac{-x}{a}g(-a) + \left(1 + \frac{x}{a}\right)g(0) \leq \frac{-x}{a}
\]
since \( g(-a) \leq 1 \) and \( g(0) \leq 0 \). For \( x \geq 0 \), \( g(x) \leq 0 \). Hence \( g(x) \leq h(x) \) for all \( x \in D \).

\textbf{3. Epigraph Characterization:}  
The epigraph \( \mathrm{epi}(h) \) coincides with the convex hull of \( \mathrm{epi}(\mathcal{L}_{0\text{-}1}) \cap (D \times \mathbb{R}) \). The affine segment \( h(x) = -\frac{x}{a} \) on \( [-a,0) \) connects the points \( (-a,1) \) and \( (0,0) \), forming the tightest convex fit to the 0-1 loss's discontinuity. By Theorem 1 in \cite{ORF523_Lec8}, this construction achieves the convex envelope. \qedhere
\end{proof}

\subsection{Proof of Corollary \ref{cor:opt_preserve}}
\optpreserve*
\begin{proof}
\textbf{Part 1:} By Theorem 1 in the lecture notes (Page 5), for any function \( f \) and convex set \( S \):
\[
\min_{x \in S} f(x) = \min_{x \in S} \mathrm{conv}_S f(x).
\]
Let \( S = D \) and \( f = \mathcal{L}_{0\text{-}1} \). The equality of minima implies:
\[
\{ x^* \in D \mid \mathcal{L}_{0\text{-}1}(x^*) = \min \mathcal{L}_{0\text{-}1} \} \subseteq \{ x^* \in D \mid \mathrm{conv}_D \mathcal{L}_{0\text{-}1}(x^*) = \min \mathrm{conv}_D \mathcal{L}_{0\text{-}1} \}.
\]
To show reverse inclusion, suppose \( x^* \in \arg\min \mathrm{conv}_D \mathcal{L}_{0\text{-}1} \). Since \( \mathrm{conv}_D \mathcal{L}_{0\text{-}1}(x^*) \leq \mathcal{L}_{0\text{-}1}(x^*) \) and \( \mathrm{conv}_D \mathcal{L}_{0\text{-}1} \) attains its minimum at the same points as \( \mathcal{L}_{0\text{-}1} \), \( x^* \) must also minimize \( \mathcal{L}_{0\text{-}1} \).

\textbf{Part 2:} For \( D = [-a, b] \), both \( \mathcal{L}_{0\text{-}1}(x) \) and \( \frac{1}{a}\mathrm{ReLU}(-x) \) attain their minimum value 0 on \( [0, b] \). For \( x \in [-a, 0) \), \( \frac{1}{a}\mathrm{ReLU}(-x) \) is strictly decreasing, achieving its minimum at \( x = 0 \). Thus:
\[
\arg\min_{x \in D} \mathcal{L}_{0\text{-}1}(x) = \arg\min_{x \in D} \frac{1}{a}\mathrm{ReLU}(-x) = [0, b].
\]
\qedhere
\end{proof}

\subsection{Proof of Corollary \ref{cor:non_convex}}
\nonconvex*
\begin{proof}
    We demonstrate violation of the convex envelope's defining property. Consider \( D = [-1,1] \):

1. \textbf{Underestimation Failure:} For \( x > 0 \):
\[
-\log \sigma(x) = -\log\left(\frac{1}{1+e^{-x}}\right) = \log(1 + e^{-x}) > 0 = \mathcal{L}_{0\text{-}1}(x)
\]
Thus \( f_{\mathrm{log-sigmoid}} \not\leq \mathcal{L}_{0\text{-}1} \) over \( D \), violating the envelope requirement.

2. \textbf{Non-Maximality:} Even if scaled, the logistic loss's curvature differs from the ReLU envelope. For \( x \in (-1,0) \), \( \frac{d^2}{dx^2}(-\log \sigma(x)) = \sigma(x)(1-\sigma(x)) > 0 \), making it strictly convex -- incompatible with the affine structure of \( \mathrm{conv}_D \mathcal{L}_{0\text{-}1} \).

Hence \( f_{\mathrm{log-sigmoid}} \) cannot be the convex envelope.
\end{proof}

\section{Experiments}

\subsection{Implementation Details}
\label{sec:appendix_implement}

Empirical observations indicate significant performance sensitivity to model parameter initialization and learning rate selection across compared methods. To establish rigorous comparison benchmarks, we conducted systematic hyperparameter searches adhering to the specifications in each method's original publication. The complete search space configuration is documented in Table~\ref{tab:hyperparams_baseline}. Notably, substantial architecture updates to both Llama3-8B and Instruct-7B necessitated re-implementation of the SimPO method, as the original implementation became incompatible with the revised model interfaces.

\textbf{Training Protocol}
All experiments employed standardized training configurations to ensure comparability: 
\begin{itemize}
    \item Batch size: 128 (consistent across methods)
    \item Learning rate: Searched in \{3\text{e-}7, 5\text{e-}7, 8\text{e-}7, 1\text{e-}6\}
    \item Training duration: Single epoch with cosine annealing schedule
    \item Warmup: 10\% of total training steps
    \item Optimizer: Adam~\citep{kingma2014adam} ($\beta_1=0.9$, $\beta_2=0.999$)
    \item Sequence length: 2048 tokens (fixed for all inputs)
\end{itemize}

The learning rate schedule follows a triangular policy with amplitude decay, selected through cross-validation on held-out development sets. All implementations utilize full-precision floating-point arithmetic to prevent gradient quantization artifacts.

\begin{table*}[t]
    \vspace{-1em}
    \caption{Various preference optimization objectives and hyperparameter search range.}
    \label{tab:hyperparams_baseline}
    \centering
    \resizebox{\textwidth}{!}{
    \small
    \begin{tabular}{lll}
    \toprule 
    \textbf{Method} & \textbf{Objective} & \textbf{Hyperparameter} \\ \midrule
    \multirow{2}{*}{SLiC-HF~\cite{SlicHF}} & \multirow{2}{*}{$\max\left(0, \delta - \log \pi_\theta(y_w|x) + \log \pi_\theta(y_l|x)\right) - \lambda \log \pi_\theta (y_w | x)$} & $\lambda \in [0.1, 0.5, 1.0, 10.0]$ \\
&  & $\delta \in [0.1, 0.5, 1.0, 2.0]$ \\ \midrule
    DPO~\citep{DPO} & $-\log \sigma \left( \beta \log \frac{\pi_\theta(y_w|x)}{\pi_{\text{ref}}(y_w|x)} - \beta \log \frac{\pi_\theta(y_l|x)}{\pi_{\text{ref}}(y_l|x)}\right)$ & $\beta \in [0.01, 0.05, 0.1]$ \\ \midrule 
    IPO~\citep{Azar2023AGT} & $ \left( \log \frac{\pi_\theta(y_w|x)}{\pi_{\text{ref}}(y_w|x)} - \log \frac{\pi_\theta(y_l|x)}{\pi_{\text{ref}}(y_l|x)} - \frac{1}{2\tau} \right)^2$ & $\tau \in [0.01, 0.1, 0.5, 1.0]$ \\  \midrule 
    CPO~\citep{xu2024contrastive} &  $-\log \sigma  \left(\beta \log \pi_\theta(y_w|x) - \beta \log \pi_\theta(y_l|x) \right) - \lambda \log \pi_\theta (y_w|x)$ & $\alpha = 1.0, \,\, \beta \in [0.01, 0.05, 0.1]$ \\ \midrule
    \multirow{2}{*}{KTO~\citep{Ethayarajh2024KTOMA}} & $-\lambda_w \sigma \left( \beta \log \frac{\pi_\theta(y_w|x)}{\pi_{\text{ref}}(y_w|x)} - z_{\text{ref}} \right) +  \lambda_l \sigma \left( z_{\text{ref}} - \beta \log \frac{\pi_\theta(y_l|x)}{\pi_{\text{ref}}(y_l|x)} \right),\,$ & $\lambda_l = \lambda_w = 1.0$ \\  
    & $\text{where} \,\, z_{\text{ref}} = \mathbb{E}_{(x, y) \sim \mathcal{D}} \left[\beta \text{KL}\left( \pi_\theta(y|x) || \pi_{\text{ref}}(y|x) \right)  \right]$ & $\beta \in [0.01, 0.05, 0.1]$ \\ \midrule
    \multirow{2}{*}{ORPO~\citep{Hong2024ORPOMP}} & $-\log p_\theta(y_w|x) - \lambda  \log \sigma \left(\log \frac{p_\theta(y_w|x)}{1 - p_\theta(y_w|x)} - \log \frac{p_\theta(y_l|x)}{1 - p_\theta(y_l|x)}  \right),\,$ & \multirow{2}{*}{$\lambda \in [0.1, 0.5, 1.0, 2.0]$} \\  
    & $\text{where} \,\, p_\theta(y|x) = \exp\left( \frac{1}{|y|} \log \pi_\theta(y|x) \right)$ \\  \midrule
    \multirow{2}{*}{R-DPO~\citep{Park2024DisentanglingLF}} & \multirow{2}{*}{$-\log \sigma \left( \beta \log \frac{\pi_\theta(y_w|x)}{\pi_{\text{ref}}(y_w|x)} - \beta \log \frac{\pi_\theta(y_l|x)}{\pi_{\text{ref}}(y_l|x)} - \left(\alpha |y_w| - \alpha |y_l| \right) \right)$} & $\alpha \in [0.05, 0.1, 0.5, 1.0]$ \\
    & & $\beta \in [0.01, 0.05, 0.1]$ \\
    \midrule 
    \multirow{2}{*}{SimPO~\citep{SimPO}} & \multirow{2}{*}{$-\log \sigma  \left( \frac{\beta}{|y_w|} \log \pi_\theta(y_w|x) - \frac{\beta}{|y_l|} \log \pi_\theta(y_l|x) - \gamma \right)$} & $\beta \in [2.0, 4.0, 6.0, 8.0]$ \\
    & & $\gamma \in [0.3, 0.5, 1.0, 1.2, 1.4, 1.6]$ \\
    \midrule 
    \method & {$\text{ReLU}[ - (\frac{1}{|y_w|} \log \pi_\theta(y_w|x) - \frac{1}{|y_l|} \log \pi_\theta(y_l|x) - \gamma)] $} & $\gamma \in [0.2, 0.4, 0.5, 0.6, 0.8]$ \\
    \bottomrule
    \end{tabular}
    }
\end{table*}

\begin{table*}[t]
\vspace{-1.3em}
\caption{The hyperparameter values in \method used for each training setting.}
\label{tab:hyperparams_simpo}
\centering
\begin{tabular}{lcc}
\toprule 
\textbf{Setting} & $\gamma$ &  Learning rate \\ \midrule
\textbf{Mistral-Instruct}   & 0.4  & 6e-7 \\
\textbf{Llama3-Instruct}  & 0.6  & 1e-6 \\
\textbf{Llama3-Instruct-v0.2}  & 0.6  & 1e-6 \\
\textbf{Gemma2-Instruct} & 0.4 & 8e-7 \\
\bottomrule
\end{tabular}
\end{table*}

\textbf{Hyperparameters in \method{}.} 
Table~\ref{tab:hyperparams_simpo} summarizes the hyperparameters utilized for \method{} across different experimental settings. Our methodology only involves  one hyperparameter: \(\gamma\). Based on empirical evidence, we recommend setting \(\gamma\) to a default value of \(0.5\), as this configuration has consistently demonstrated reliability. 

\textbf{Decoding Hyperparameters.} 
The decoding hyperparameters employed in this study align with those used in SimPO\footnote{\url{https://github.com/princeton-nlp/SimPO/tree/main/eval}}. We express our gratitude to the SimPO team for their generosity in sharing their insights and configurations, which have been instrumental in our work.

\textbf{Computation Environment.} 
All training experiments described in this paper were conducted using 8×A100 GPUs. The experimental setup follows the guidelines provided in the alignment-handbook repository\footnote{\url{https://github.com/huggingface/alignment-handbook}}, ensuring reproducibility and consistency with established practices.

\subsection{Downstream Task Evaluation}
\label{sec:downstram_task}
\begin{table}[h]
    \caption{Downstream task evaluation results of tasks on the huggingface open leaderboard. \label{tab:downstream_v2}}
    \resizebox{\textwidth}{!}{\begin{tabular}{@{}lccccccc@{}}
    \toprule
                   & \textbf{MMLU (5)} & \textbf{ARC (25)} & \textbf{HellaSwag (10)} & \textbf{TruthfulQA (0)} & \textbf{Winograd (5)} & \textbf{GSM8K (5)} & \textbf{Average} \\ \midrule
                   \multicolumn{8}{c}{{\color[HTML]{222222} \textbf{Llama3-Instruct}}}                                                                                       \\ \midrule
                   \textbf{SFT}   & 67.06            & 61.01            & 78.57                  & 51.66                  & 74.35                & 68.69             & 66.89           \\
                   \textbf{RRHF} & 67.20 & 61.52 & 79.54 & 53.76 & 74.19 & 66.11 & 67.05 \\
                   \textbf{SLiC-HF} & 66.41 & 61.26 & 78.80 & 53.23 & 76.16 & 66.57 & 67.07 \\
                   \textbf{DPO}   & 66.88            & 63.99            & 80.78                  & 59.01                  & 74.66                & 49.81             & 65.86           \\
                   \textbf{IPO}   & 66.52            & 61.95            & 77.90                  & 54.64                  & 73.09                & 58.23             & 65.39           \\
                   \textbf{CPO} & 67.05 & 62.29 & 78.73 & 54.01 & 73.72 & 67.40 & 67.20 \\
                   \textbf{KTO}   & 66.38            & 63.57            & 79.51                  & 58.15                  & 73.40                & 57.01             & 66.34           \\
                   \textbf{ORPO}  & 66.41            & 61.01            & 79.38                  & 54.37                  & 75.77                & 64.59             & 66.92           \\
                   \textbf{R-DPO} & 66.74            & 64.33            & 80.97                  & 60.32                  & 74.82                & 43.90             & 65.18           \\
                   \textbf{SimPO} & 65.63            & 62.80            & 78.33                  & 60.70                  & 73.32                & 50.72             & 65.25           \\ 
                   \textbf{\method{}} & 64.95      &62.03      &77.58      &60.96      &72.93      &66.49      &67.49 \\ \midrule
                   \multicolumn{8}{c}{{\color[HTML]{222222} \textbf{Llama3-Instruct v0.2}}}                                                                            \\ \midrule
                   \textbf{SFT}   & 67.06            & 61.01            & 78.57                  & 51.66                  & 74.35                & 68.69             & 66.89           \\
                   \textbf{RRHF} & 66.60 & 63.74 & 80.98 & 59.40 & 76.32 & 58.68 & 67.62 \\
                   \textbf{SLiC-HF} &  66.91 & 61.77 & 79.17 & 56.36 & 76.40 & 68.23 & 68.14 \\
               
                   \textbf{DPO}   & 65.57      &65.87      &79.66      &63.08      &74.51      &73.01      &70.28 \\
                   \textbf{IPO}   & 66.06      &64.85      &81.02      &57.29      &76.72      &76.12      &70.34 \\
                   \textbf{CPO} & 65.67      &62.12      &79.63      &56.34      &77.98      &75.28      &69.50 \\
                   \textbf{KTO}   & 65.99      &62.88      &79.02      &54.66      &74.66      &76.42      &68.94 \\
                   \textbf{ORPO}  & 65.75      &63.99      &79.91      &57.02      &78.06      &75.13      &69.98 \\
                   \textbf{R-DPO} & 66.17      &65.36      &79.98      &57.94      &75.06      &75.36      &69.98           \\
                   \textbf{SimPO} & 65.18      &67.15      &78.04      &64.92      &73.88      &71.34      &70.08    \\
                   \textbf{\method{}} & 65.00      &68.09      &80.50      &64.38      &76.16      &69.37      &70.58  \\ \bottomrule
                \end{tabular}}
\end{table}

To assess the impact of \method{} on downstream task performance, we evaluate models trained with different preference optimization methods on a diverse set of tasks from the Huggingface Open Leaderboard~\cite{open-llm-leaderboard}. The tasks include MMLU~\cite{hendrycks2020measuring}, ARC~\cite{clark2018think}, HellaSwag~\cite{zellers-etal-2019-hellaswag}, TruthfulQA~\cite{lin2022truthfulqa}, Winograd~\cite{levesque2012winograd}, and GSM8K~\cite{cobbe2021gsm8k}. We adhere to standard evaluation protocols and present the results for all models in \Cref{tab:downstream_v2}.

\paragraph{Overall Performance.}
On average, \method{} shows competitive performance across tasks, achieving an overall score of 67.49 on the Llama3-Instruct model and 70.58 on the Llama3-Instruct v0.2 model. The performance is generally close to that of other preference optimization methods, but it is worth noting that in some cases, it slightly lags behind models like SimPO or DPO, particularly on tasks such as ARC, HellaSwag, and TruthfulQA. However, the results suggest that \method{} maintains a balanced performance profile across the evaluated tasks.

\textbf{General Knowledge and Reasoning.}
On MMLU, which tests general knowledge and reasoning, \method{} shows a slight reduction in performance (64.95 for Llama3-Instruct and 65.00 for Llama3-Instruct v0.2) compared to models such as RRHF and SimPO. This minor decline is consistent with the trend observed for other preference optimization methods and indicates that \method{} may preserve general knowledge to a similar extent while possibly focusing more on improving performance in other areas such as reading comprehension and reasoning.

\textbf{Reading Comprehension and Commonsense Reasoning.}
For ARC and HellaSwag, tasks related to reading comprehension and commonsense reasoning, \method{} outperforms the base SFT model and exhibits competitive performance relative to other preference optimization methods. The Llama3-Instruct v0.2 model with \method{} achieves a score of 80.50 on HellaSwag, which is comparable to the best-performing methods. This result suggests that \method{} effectively improves the model’s ability to handle contextual understanding and reasoning, likely due to its optimization strategy.

\textbf{Truthfulness.}
On the TruthfulQA task, \method{} consistently shows improvements over the base SFT model, with a score of 60.96 for Llama3-Instruct and 64.38 for Llama3-Instruct v0.2. This indicates that \method{} helps the model generate more truthful and reliable responses, aligning with trends observed in other preference optimization methods. The improvement in this area is especially notable given the inherent difficulty of this task, which tests the model’s ability to avoid generating false information.

\textbf{Math Performance.}
The GSM8K benchmark, which tests mathematical reasoning, shows a drop in performance for \method{} relative to the base SFT model. Specifically, the Llama3-Instruct model with \method{} achieves a score of 66.49, which is lower than other methods such as SimPO or R-DPO, which focus more on improving mathematical reasoning. This drop is consistent with the trend observed across various preference optimization methods and may suggest that \method{} is less effective in retaining mathematical reasoning abilities. Further investigation into this issue could provide insights into potential strategies for addressing this gap.

\textbf{Task-Specific Variability.}
Overall, \method{} exhibits varied performance across tasks. While it performs well in certain areas, such as commonsense reasoning and truthfulness, it lags behind in others, particularly in general knowledge (MMLU) and mathematical reasoning (GSM8K). This variability is in line with the performance trends observed for other preference optimization methods, which often show task-dependent improvements and declines. This suggests that \method{} has strengths in some domains, but it may benefit from further refinement to improve performance across all tasks.


\subsection{\method{} with varying $\gamma$}
\label{sec_appendix_vary_gamma}
\begin{figure}[t]
    \begin{center}
        \includegraphics[width=0.6\columnwidth]{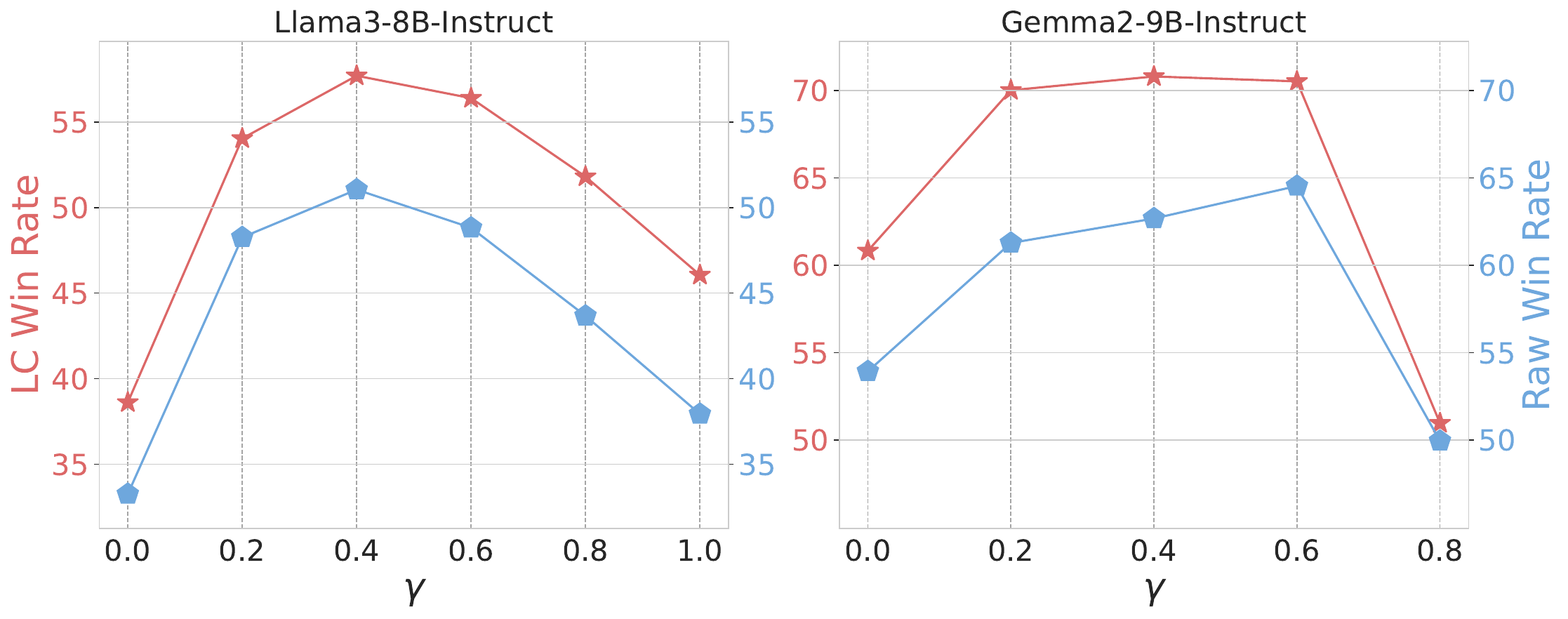}
        \caption{Impact of the hyperparameter \(\gamma\) on LC Win Rate and Raw Win Rate for Llama3-8B-Instruct (left) and Gemma2-9B-Instruct (right). 
        }
        \label{fig:vary_gamma}
    \end{center}
\end{figure}

Figure~\ref{fig:vary_gamma} illustrates the effect of the hyperparameter \(\gamma\) on model performance across two evaluation metrics: LC Win Rate and Raw Win Rate. The analysis is conducted on two models, Llama3-8B-Instruct (left) and Gemma2-9B-Instruct (right). The LC Win Rate, shown in red (left y-axis), represents the model’s alignment with learned preferences, whereas the Raw Win Rate, shown in blue (right y-axis), evaluates overall ranking performance based on human preference comparisons.

\textbf{Moderate \(\gamma\) values lead to optimal performance.} 
Moderate values of \(\gamma\) (0.4–0.6) yield the best balance between preference alignment and generalization. Both models achieve their highest LC Win Rate in this range, indicating that preference optimization is most effective when applied at an intermediate level. As \(\gamma\) increases beyond 0.6, LC Win Rate starts to decline, likely due to overfitting, where the model overly aligns with preference data at the expense of generalization. Conversely, at \(\gamma=0.0\), where no preference optimization is applied, the LC Win Rate remains low, emphasizing the necessity of preference tuning.

\textbf{Raw Win Rate trends reveal model robustness differences.}
The Raw Win Rate follows a similar trend but highlights differences in robustness across models. For Llama3-8B-Instruct, the Raw Win Rate peaks at \(\gamma=0.4\) before declining, suggesting that excessive preference optimization (\(\gamma>0.6\)) negatively impacts the model’s ability to generalize. In contrast, Gemma2-9B-Instruct exhibits a more stable Raw Win Rate across a wider range of \(\gamma\), reaching its highest performance at \(\gamma=0.6\) before experiencing a sharp decline at \(\gamma=0.8\). This suggests that Gemma2-9B-Instruct maintains better robustness to preference optimization compared to Llama3-8B-Instruct.

\textbf{Gemma2-9B-Instruct outperforms Llama3-8B-Instruct.} 
Gemma2-9B-Instruct consistently outperforms Llama3-8B-Instruct in both LC Win Rate and Raw Win Rate. This observation indicates that Gemma2-9B-Instruct not only aligns more effectively with learned preferences but also retains superior generalization capability. The results highlight the importance of carefully selecting \(\gamma\) to avoid performance degradation at extreme values. Future work could explore adaptive strategies for dynamically tuning \(\gamma\), ensuring that preference optimization enhances alignment without compromising generalization.

\subsection{Analsis on \newmethod{}}
Figure \ref{fig:gradient_analysis_ours} illustrates the relationship between the gradient and the implicit reward margin. As shown in the figure, when the implicit reward margin is greater than $\gamma$, the gradient becomes zero. In this case, the model can stop updating for well-separated pairs, thus preventing overfitting. On the other hand, when the implicit reward margin is less than $\gamma$, the model continues to increase the weight for less-separated pairs. Furthermore, the harder the pair is to distinguish, the larger the gradient becomes, eventually converging to 1.0. This behavior is reminiscent of curriculum learning, where more difficult samples are assigned higher weights.

\begin{figure}[t]
    \begin{center}
        \includegraphics[width=0.4\columnwidth]{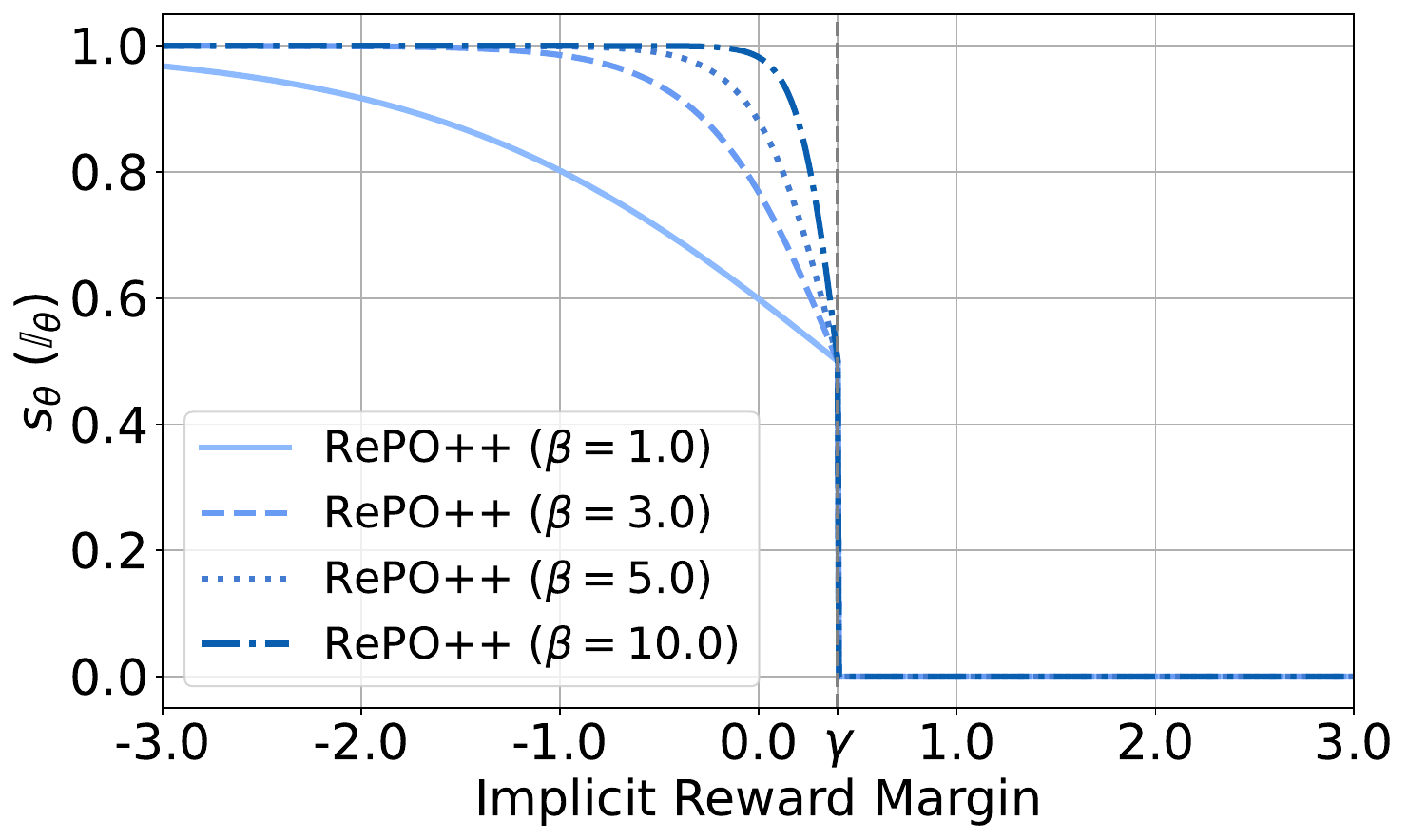}
        \caption{Gradient weighting functions of \newmethod{} ($s_\theta \cdot \mathbb{I}(M_\theta < \gamma)$).}
        \label{fig:gradient_analysis_ours}
    \end{center}
\end{figure}

\subsection{Analysis of the Relationship Between SLiC-HF and RePO}
\label{sec:appendix_slichf_repo}

To provide deeper insights into the relationship between SLiC-HF and RePO, we conducted additional experiments examining the effect of SFT regularization—a core component of SLiC-HF that is absent in our method. As demonstrated in Table \ref{tab:hyperparams_slichf}, we systematically evaluated performance across varying values of the regularization coefficient $\lambda$.

\textbf{Mathematical Comparison.} From a formulation perspective, SLiC-HF combines a hinge loss term with an SFT regularization component that penalizes deviation from the reference model. Specifically, the SFT regularization is controlled by parameter $\lambda$, which balances preference optimization against model drift. By contrast, RePO eliminates this regularization entirely, relying solely on its binary threshold mechanism to control optimization.

\textbf{Impact of SFT Regularization.} Our experimental results reveal a consistent trend: as $\lambda$ increases from 0.0 to 5.0, performance on AlpacaEval2 LC steadily declines from 34.1\% to 27.8\%. This progressive degradation suggests that stronger regularization toward the initial SFT model actually hinders effective preference learning in this context. The optimal performance occurs at $\lambda=0.0$, which effectively transforms SLiC-HF into a variant of RePO.

\textbf{Different Optimization Challenges.} These findings suggest that SFT regularization and RePO's threshold-based filtering address fundamentally different optimization challenges. While SFT regularization was originally introduced to mitigate catastrophic forgetting and preserve general capabilities, our results indicate that for direct preference optimization, such regularization is unnecessary and potentially counterproductive. Instead, RePO's selective gradient application through its threshold mechanism appears sufficient to prevent over-optimization while maintaining effective preference learning.

This analysis complements our main findings and further supports our hypothesis that carefully designed filtering mechanisms can effectively replace more complex regularization schemes in preference optimization.

\begin{table*}[h!]
\vspace{-1.3em}
\caption{The hyperparameter $\lambda$ in SliC-HF used for each Llama3-Instruct v0.2.}
\label{tab:hyperparams_slichf}
\centering
\small
\begin{tabular}{lcccccc}
\toprule 
\textbf{SLiC-HF} & $\lambda=0.0$ & $\lambda=0.1$ & $\lambda=0.5$ & $\lambda=1.0$ & $\lambda=3.0$ & $\lambda=5.0$  \\ 
\midrule
AlpacaEval2 LC  & 34.1 & 33.9 & 32.8 & 30.8 & 28.6 & 27.8 \\
\bottomrule
\end{tabular}
\end{table*}

\subsection{Detailed Analysis of the Relationship Between DPO and RePO}
\label{sec:appendix_dpo_repo}

To further investigate the relationship between DPO and RePO, we conducted additional experiments examining how different formulation components affect performance. As shown in Table \ref{tab:ablation2}, we systematically evaluated five variants that decompose the key elements of each approach.

\textbf{DPO as a Special Case of SimPO.} From a mathematical perspective, DPO can be viewed as a specific instance of SimPO where $\gamma = \log\pi_{\text{ref}}(y_w \mid x) - \log\pi_{\text{ref}}(y_l \mid x)$ (ignoring length normalization for equivalence). This connection highlights how DPO implicitly defines its target margin based on reference model probabilities rather than using an explicit hyperparameter.

\textbf{Impact of ReLU Without Explicit Margin.} The second row of Table \ref{tab:ablation2} shows that directly replacing log-sigmoid with ReLU while maintaining DPO's implicit margin definition leads to catastrophic performance degradation (-44\% on AlpacaEval LC, -47\% on AlpacaEval WR, and -26\% on Arena-Hard). This dramatic decline reveals that the binary threshold mechanism of ReLU is only effective when paired with an appropriate explicit margin parameter.

\textbf{The Critical Role of $\gamma$.} Rows 3-5 demonstrate that adding an explicit $\gamma$ parameter consistently improves performance across all metrics regardless of whether log-sigmoid, ReLU, or their combination is used. The most substantial gains appear when ReLU and $\gamma$ are combined (+4.4\% LC, +9.1\% WR on AlpacaEval), supporting our hypothesis that explicit threshold-based filtering effectively controls over-optimization.

\textbf{Complementary Mechanisms.} Interestingly, the combination of both mechanisms (row 5) yields the highest overall performance, suggesting that while RePO's binary filtering mechanism addresses the core over-optimization challenge, the continuous weighting from log-sigmoid may provide complementary benefits for fine-grained preference learning.

\begin{table}[!h]
    \caption{
        Impact Analysis of $\gamma$ Scaling and ReLU Mechanisms in DPO Training. Benchmark results on AlpacaEval 2.0 (AE2) and Arena-Hard (AH) demonstrate percentage point changes in Length-Controlled Win Rate (LC-WR) and Base Win Rate (WR) for Llama3-Instruct-v0.2 (8B). Values represent relative performance deltas (\%) compared to standard DPO baseline.}
    \label{tab:ablation2}
    \vspace{-10pt}
    \begin{center}
    \resizebox{1.0\linewidth}{!}{%
    \begin{tabular}{lllllll}
    \toprule
    \multirow{3}{*}{\textbf{Method}} & \multirow{3}{*}{\textbf{$\gamma$}} & \multirow{3}{*}{\textbf{ReLU}} & \multirow{3}{*}{\textbf{$\log \sigma$}} & \multicolumn{3}{c}{\textbf{Llama3-Instruct v0.2 (8B)}} \\
    \cmidrule(lr){5-7}
    &&& & \multicolumn{2}{c}{\textbf{AE 2}} & \multicolumn{1}{c}{\textbf{AH}} \\
    \cmidrule(lr){5-6} \cmidrule(lr){7-7}
    & & & & \textbf{LC} & \textbf{WR} &  \textbf{WR}\\
    \midrule
    $-\log \sigma \left( \beta \log \frac{\pi_\theta(y_w|x)}{\pi_{\text{ref}}(y_w|x)} - \beta \log \frac{\pi_\theta(y_l|x)}{\pi_{\text{ref}}(y_l|x)}\right)$ 
    &\ding{53} & \ding{53}& \ding{51}&48.2 & 47.5  & 35.2  \\
    $\text{ReLU} \left( - \left( \log \frac{\pi_\theta(y_w|x)}{\pi_{\text{ref}}(y_w|x)} - \log \frac{\pi_\theta(y_l|x)}{\pi_{\text{ref}}(y_l|x)} \right) \right)$ 
    & \ding{53} & \ding{51} & \ding{53} &$26.9^{\color{-}-44\%}$  & $25.1^{\color{-}-47\%}$ & $26.2^{\color{-}-26\%}$  \\
    \midrule \midrule
    $-\log \sigma \left( \beta \log \frac{\pi_\theta(y_w|x)}{\pi_{\text{ref}}(y_w|x)} - \beta \log \frac{\pi_\theta(y_l|x)}{\pi_{\text{ref}}(y_l|x)} - \gamma \right)$
    &\ding{51} & \ding{53} & \ding{51}&$50.0^{\color{+}+3.7\%}$  & $50.7^{\color{+}+6.7\%}$ & $36.8^{\color{+}+4.5\%}$  \\
    $\text{ReLU} \left( - \left(  \log \frac{\pi_\theta(y_w|x)}{\pi_{\text{ref}}(y_w|x)} - \log \frac{\pi_\theta(y_l|x)}{\pi_{\text{ref}}(y_l|x)} - \gamma \right) \right)$ 
    &\ding{51} &\ding{51} & \ding{53}&$50.3^{\color{+}+4.4\%}$  & $51.8^{\color{+}+9.1\%}$ & $38.2^{\color{+}+8.5\%}$  \\
    $- \log \sigma \left( - \text{ReLU} \left( - \left( \beta \log \frac{\pi_\theta(y_w|x)}{\pi_{\text{ref}}(y_w|x)} - \beta \log \frac{\pi_\theta(y_l|x)}{\pi_{\text{ref}}(y_l|x)} - \gamma \right) \right) \right)$ 
    &\ding{51} &\ding{51} & \ding{51}&$50.8^{\color{+}+5.4\%}$ & $52.2^{\color{+}+9.9\%}$  & $37.2^{\color{+}+5.7\%}$  \\
    \bottomrule
    \end{tabular}
    }
    \vspace{-15pt}
    \end{center}
\end{table}


\newpage
\section*{NeurIPS Paper Checklist}

\begin{enumerate}

\item {\bf Claims}
    \item[] Question: Do the main claims made in the abstract and introduction accurately reflect the paper's contributions and scope?
    \item[] Answer: \answerYes{} 
    \item[] Justification: see abstract and introduction.

\item {\bf Limitations}
    \item[] Question: Does the paper discuss the limitations of the work performed by the authors?
    \item[] Answer: \answerYes{} 
    \item[] Justification: see Section \ref{sec:conclusion}.

\item {\bf Theory assumptions and proofs}
    \item[] Question: For each theoretical result, does the paper provide the full set of assumptions and a complete (and correct) proof?
    \item[] Answer: \answerYes{} 
    \item[] Justification: see Appendix \ref{sec:proofs}.

    \item {\bf Experimental result reproducibility}
    \item[] Question: Does the paper fully disclose all the information needed to reproduce the main experimental results of the paper to the extent that it affects the main claims and/or conclusions of the paper (regardless of whether the code and data are provided or not)?
    \item[] Answer: \answerYes{} 
    \item[] Justification: see Appendix \ref{sec:appendix_implement}.

\item {\bf Open access to data and code}
    \item[] Question: Does the paper provide open access to the data and code, with sufficient instructions to faithfully reproduce the main experimental results, as described in supplemental material?
    \item[] Answer: \answerYes{} 
    \item[] Justification: \url{https://github.com/junkangwu/RePO}.

\item {\bf Experimental setting/details}
    \item[] Question: Does the paper specify all the training and test details (e.g., data splits, hyperparameters, how they were chosen, type of optimizer, etc.) necessary to understand the results?
    \item[] Answer: \answerYes{} 
    \item[] Justification: see Appendix \ref{sec:appendix_implement}.

\item {\bf Experiment statistical significance}
    \item[] Question: Does the paper report error bars suitably and correctly defined or other appropriate information about the statistical significance of the experiments?
    \item[] Answer: \answerYes{} 
    \item[] Justification: The results are accompanied by error bars, confidence intervals, or statistical
significance tests, at least for the experiments that support the main claims of the paper.

\item {\bf Experiments compute resources}
    \item[] Question: For each experiment, does the paper provide sufficient information on the computer resources (type of compute workers, memory, time of execution) needed to reproduce the experiments?
    \item[] Answer: \answerYes{} 
    \item[] Justification: We carried out all computational tasks on a suite of four 80GB A100 GPUs.
    
\item {\bf Code of ethics}
    \item[] Question: Does the research conducted in the paper conform, in every respect, with the NeurIPS Code of Ethics \url{https://neurips.cc/public/EthicsGuidelines}?
    \item[] Answer: \answerYes{} 
    \item[] Justification: Our research has been conducted with strict adherence to the NeurIPS Code of
Ethics.

\item {\bf Broader impacts}
    \item[] Question: Does the paper discuss both potential positive societal impacts and negative societal impacts of the work performed?
    \item[] Answer: \answerYes{} 
    \item[] Justification: see Appendix \ref{sec:broader}.
    
\item {\bf Safeguards}
    \item[] Question: Does the paper describe safeguards that have been put in place for responsible release of data or models that have a high risk for misuse (e.g., pretrained language models, image generators, or scraped datasets)?
    \item[] Answer: \answerNA{} 
    \item[] Justification: The paper poses no such risks.

\item {\bf Licenses for existing assets}
    \item[] Question: Are the creators or original owners of assets (e.g., code, data, models), used in the paper, properly credited and are the license and terms of use explicitly mentioned and properly respected?
    \item[] Answer: \answerNA{} 
    \item[] Justification: The paper does not use existing assets.

\item {\bf New assets}
    \item[] Question: Are new assets introduced in the paper well documented and is the documentation provided alongside the assets?
    \item[] Answer: \answerNA{} 
    \item[] Justification: The paper does not release new assets.

\item {\bf Crowdsourcing and research with human subjects}
    \item[] Question: For crowdsourcing experiments and research with human subjects, does the paper include the full text of instructions given to participants and screenshots, if applicable, as well as details about compensation (if any)? 
    \item[] Answer: \answerNA{} 
    \item[] Justification: The paper does not involve crowdsourcing nor research with human subjects.

\item {\bf Institutional review board (IRB) approvals or equivalent for research with human subjects}
    \item[] Question: Does the paper describe potential risks incurred by study participants, whether such risks were disclosed to the subjects, and whether Institutional Review Board (IRB) approvals (or an equivalent approval/review based on the requirements of your country or institution) were obtained?
    \item[] Answer: \answerNA{} 
    \item[] Justification: The paper does not involve crowdsourcing nor research with human subjects.

\item {\bf Declaration of LLM usage}
    \item[] Question: Does the paper describe the usage of LLMs if it is an important, original, or non-standard component of the core methods in this research? Note that if the LLM is used only for writing, editing, or formatting purposes and does not impact the core methodology, scientific rigorousness, or originality of the research, declaration is not required.
    \item[] Answer: \answerNA{} 
    \item[] Justification: The LLM is used only for writing, editing.

\end{enumerate}

\end{document}